\documentclass[11pt]{article}
\usepackage{fullpage}
\date{}

\title{\bfseries\papertitle}%\\in the Plackett-Luce model} 
\author{
Aadirupa Saha%
\thanks{Microsoft Research, New York City, USA; {\tt aadirupa.saha@microsoft.com}.}
\and 
Pierre Gaillard \thanks{Inria Greenoble - Rhône-Alpes and the Laboratoire Jean Kuntzmann, France. {\tt pierre.gaillard@inria.fr}}
}
\usepackage[T1]{fontenc}
\usepackage{lmodern}
\usepackage{hyperref}       % hyperlinks
\usepackage{url}            % simple URL typesetting
\usepackage{booktabs}       % professional-quality tables
\usepackage{amsfonts}       % blackboard math symbols
\usepackage{nicefrac}       % compact symbols for 1/2, etc.
\usepackage{microtype}      % microtypography
\usepackage{xcolor}         % colors
\usepackage{enumitem}
\usepackage{microtype}
\usepackage{amssymb}
\usepackage{pifont}
\usepackage{nicefrac}
\usepackage{cancel}
\usepackage{array}
\usepackage{graphicx}
\usepackage{booktabs} % for professional tables
\usepackage{amssymb}
\usepackage{color}
\usepackage{amsmath}
\usepackage{amsthm}
\usepackage{thmtools}
\usepackage{thm-restate}
\usepackage{algorithm}
\usepackage{algorithmic}
\usepackage{hyperref}
\usepackage{nameref}
\newtheorem{thm}{Theorem}%[section]
\newtheorem*{thm*}{Theorem}
\newtheorem*{lem*}{Lemma}
\newtheorem{defn}[thm]{Definition}

\newtheorem{rem}{Remark}

\newtheorem{innercustomthm}{Theorem}
\newenvironment{customthm}[1]
  {\renewcommand\theinnercustomthm{#1}\innercustomthm}
  {\endinnercustomthm}

 \newtheorem{innercustomlemma}{Lemma}
\newenvironment{customlemma}[1]
  {\renewcommand\theinnercustomlemma{#1}\innercustomlemma}
  {\endinnercustomlemma}

\newcommand{\R}{{\mathbb R}}

\renewcommand{\P}{{\mathbf P}}

\newcommand{\W}{{\mathbf W}}
\newcommand{\E}{{\mathbf E}}

\newcommand{\1}{{\mathbf 1}}

\newcommand{\cX}{{\mathcal X}}
\newcommand{\cA}{{\mathcal A}}

\newcommand{\cB}{{\mathcal B}}

\newcommand{\cC}{{\mathcal C}}
\newcommand{\hcB}{{\hat \cB}}

\newcommand{\cI}{{\mathcal I}}

\newcommand{\cJ}{{\mathcal J}}

\newcommand{\cS}{{\mathcal S}}
\newcommand{\cE}{{\mathcal E}}

\newcommand{\cG}{{\mathcal G}}

\newcommand{\M}{{\mathbf M}}

\newcommand{\bN}{{\mathbf N}}

\newcommand{\cD}{{\mathcal D}}

\DeclareMathOperator{\kl}{kl}

\newcommand{\hp}{{\hat p}}
\newcommand{\hi}{{\hat i}}

\newcommand{\p}{{\mathbf p}}

\newcommand{\y}{{\mathbf y}}

\newcommand{\sm}{\setminus}

\renewcommand{\epsilon}{\varepsilon}
\renewcommand{\hat}{\widehat}

\newcommand{\btheta}{\boldsymbol \theta}

\newcommand{\bSigma}{\boldsymbol \Sigma}

\newcommand{\bmu}{{\boldsymbol \mu}}

\newcommand{\bsigma}{\boldsymbol \sigma}

\DeclareMathOperator*{\argmax}{argmax}

\def \algrucbp{\texttt{SlDB-UCB}}
\def \algrmed{\texttt{SlDB-ED}}
\def \dspaa{{\it Sleeping-Dueling Bandit with Stochastic Preferences and Adversarial Availabilities}}
\def \SPAA{\texttt {DB-SPAA}}
\def \papertitle{Dueling Bandits with Adversarial Sleeping} %{Sleeping Dueling Bandits}

\begin{document}

\maketitle

%!TEX root = main.tex
\begin{abstract}
We introduce the problem of sleeping dueling bandits with stochastic preferences and adversarial availabilities (DB-SPAA). In almost all dueling bandit applications, the decision space often changes over time; eg, retail store management, online shopping, restaurant recommendation, search engine optimization, etc. Surprisingly, this `sleeping aspect' of dueling bandits has never been studied in the literature. Like dueling bandits, the goal is to compete with the best arm by sequentially querying the preference feedback of item pairs. The non-triviality however results due to the non-stationary item spaces that allow any arbitrary subsets items to go unavailable every round. The goal is to find an optimal `no-regret' policy that can identify the best available item at each round, as opposed to the standard `fixed best-arm regret objective' of dueling bandits. We first derive an instance-specific lower bound for DB-SPAA $\Omega( \sum_{i =1}^{K-1}\sum_{j=i+1}^K \frac{\log T}{\Delta(i,j)})$, where $K$ is the number of items and $\Delta(i,j)$ is the gap between items $i$ and $j$. This indicates that the sleeping problem with preference feedback is inherently more difficult than that for classical multi-armed bandits (MAB).  We then propose two algorithms, with near optimal regret guarantees.
%respectively with $O(\sum_{i = 1}^{K-1}\sum_{j = i+1}^K \frac{\log T}{\Delta(i,j)})$ and $O(\min\{KT^{2/3}, K^2 \log(T)/\Delta \})$ regret ($\Delta$ is the minimal gap), justifying optimality of our analysis. 
%All in all, we present a clean tradeoff between regret-vs-availability sequence, compared to what is known for MAB setup. 
Our results are corroborated empirically.
\end{abstract}
\vspace{-10pt}

%%%%%%%%%%%%%%%%%%%%%%%%%%%%%%%%%%%%%%%%%%%%%%%%%%%%%%

%!TEX root = sldb-nips21.tex
% \vspace{-14pt}
\section{Introduction}
\label{sec:intro}
% \vspace{-5pt}
The problem of \emph{Dueling-Bandits} has gained much attention in the machine learning community \cite{Yue+12,Zoghi+14RCS,Zoghi+15}, which is an online learning framework that generalizes the standard multiarmed bandit (MAB) \cite{Auer+02} setting for identifying a set of `good' arms from a fixed decision-space (set of arms/items) by querying preference feedback of actively chosen item-pairs. More formally, in dueling bandits, the learning proceeds in rounds: At each round, the learner selects a pair of arms and observes stochastic preference feedback of the winner of the comparison (duel) between the selected arms; the objective of the learner is to minimize the regret with respect to a (or set of) `best' arm(s) in hindsight.  
Towards this several algorithms have been proposed \cite{Ailon+14,Zoghi+14RUCB,Komiyama+15,Adv_DB}.
Due to the inherent exploration-vs-exploitation tradeoff of the learning framework and several 
advantages of preference feedback \cite{Busa14survey,Yue+09}, many real-world applications can be modeled as 
dueling bandits, including movie recommendations, retail management, search engine optimization, job scheduling, etc.

\begin{figure}%{1\textwidth}
%\vspace{-\baselineskip}
	\begin{center}
		\begin{tabular}{l}
			\hline \\[-10pt]
			\textbf{Parameters. }Item set: $[K]$ (known), Preference: $\P$ (un-\\known), Available item sets: $\cS_T$ (observed sequentially)\\
			\textbf{For} $t = 1, 2, \ldots, T$, the learner: \\
		\quad \textbullet~ Observes $S_t \subseteq [K]$ the set of available items \\
		\quad \textbullet~ Chooses $(x_t, y_t) \in S_t^2$  \\
		\quad \textbullet~ Observes $o_t := \1(x_t \succ y_t) \sim \text{Ber}(\P(x_t,y_t))$  \\
		\quad \textbullet~ Incurs $r_t := \nicefrac{1}{2} \big(\P(i_t^*,x_t)+\P(i_t^*,y_t)-1\big)$; \\
		\quad\quad\quad where $i_t^*$ is such that $\min_{j \in S_t} \P(i_t^*, j) \geq \nicefrac{1}{2}$ \\
			\hline
		\end{tabular}  \vspace*{-10pt}
	\end{center}
  \caption{Setting of \SPAA$(\P,\cS_T)$} %\vspace*{-2\baselineskip}
  \label{fig:setting}
\end{figure}

% \vspace{-10pt}

%\red{Pierre: I changed $\P_K$ to $\P$ everywhere (do the same in the experiments) in order to have consistent notations\\}---\blue{Aadirupa: Done, except in lower bounds as we have a lower bound for $K =2$, makes thing explicit, but mentioned specifically in the beginning of the section.}

However, in reality, the decision spaces might often change over time due to the non-availability of some items, which are considered to be `sleeping'. This `sleeping-aspect' of online decision making problems has been widely studied in the standard multiarmed bandit (MAB) literature \cite{kanade09,neu14,kale16,kleinberg+10,kanade14,
cortes+19}. There the goal is to learn a `no-regret' policy that maps to the `best awake item' of any available (non-sleeping) subset of items, and the learner's performance is measured with respect to the optimal policy in hindsight. This setting is famously known as \emph{Sleeping Bandits} in MAB \cite{kanade09,neu14,kale16,cortes+19}. More discussions are given in Related Works.

Surprisingly, however, the \emph{`sleeping problem'} is completely unaddressed in the preference bandits literature, even for the special case of pairwise preference feedback, which is famously studied as \emph{Dueling Bandits} \cite{Zoghi+14RUCB,Yue+12}, even though the setup of changing decision spaces are quite relevant in almost every practical applications: Be that in retail stores where some items might go out of production over time, for search engine optimization some websites could be down on certain days, in recommender systems some restaurants might be closed or movies could be outdated, in clinical trials certain drugs could be out of stock, and many more. 
This work is the first to consider the problem of \emph{Sleeping Dueling Bandits}, where we formulated the stochastic $K$-armed dueling bandit problem with adversarial item availabilities. Here at each round $t \in \{1,2,\ldots, T\}$ the item preferences are considered to be generated from a fixed underlying (and of course unknown) preference matrix $\P \in [0,1]^{K \times K}$, however, the set of available actions $S_t \subseteq \{1,2,\ldots, K\}$ is assumed to be adversarially chosen by the environment. We call the problem as \dspaa\, or in brief \SPAA$(\P,\cS_T)$, where $\cS_T = \{S_1,S_2,\ldots S_T\}$ denotes the sequence of available subsets over $T$ rounds. 
We also assume the preference $\P$ follows a `total-ordering assumption to ensure the existence of a best-item per available subset $S_t$.
We describe the setting in Fig.~\ref{fig:setting} with a formal description in Sec.~\ref{sec:prob}.
%
%For the above problem, we give a computationally efficient and optimal no-regret algorithm along with a matching lower bound analysis. 
Our specific contributions are as follows:

%\noindent \textbf{Contributions: 

%\begin{itemize} %[itemsep=2pt,parsep=2pt,topsep=0pt]

\textbf{1.} We first analyze the fundamental performance limit for the \SPAA$(\P,\cS_T)$ problem in Sec.~\ref{sec:lb}:  Thm.~\ref{thm:lb} gives an instance-specific regret lower bound of 
\[
	\smash{\textstyle{\Omega\big( \sum_{i = 1}^{K-1}\sum_{j = i+1}^K \frac{\log T}{\Delta(i,j)} \big)} }\,, 
\]
with $\Delta(i,j)$ being the `preference gap' of item $i$-vs-$j$ (see Eqn.~\eqref{eq:gap}). 
Our lower bound, which can be of order $\Omega(K^2\log T/\Delta)$, $\Delta:= \min_{i,j}\Delta(i,j)$ being the worst case gap, indicates that the \emph{problem of sleeping dueling bandits is inherently more difficult that standard sleeping bandits (MAB)}, unlike the `non-sleeping' case where both dueling bandits (with `total-ordering' assumption on $\P$) and MAB are known to have the same fundamental performance limit of $\Omega(K \log T/\Delta)$ (Rem.~\ref{rem:lb}).
%Our bound indicates that the problem of sleeping dueling bandits is inherently more difficult than sleeping MAB, unlike the `non-sleeping' case where the fundamental performance limit for both the settings are known to be of the same order $O(\frac{K}{\Delta} \log T)$ assuming the dueling preferences respect a \emph{condorcet winner} \cite{Auer+02,Komiyama+15}. Thus we see that clearly, the `sleeping-aspect' of dueling bandits makes the problem $K$-times harder than the same for MAB. 

\textbf{2.} We next design a \emph{`fixed confidence regret'} algorithm \algrucbp \, (Alg.~\ref{alg:rucbp}), inspired from the pairwise upper confidence bound (UCB) based algorithm \cite{Zoghi+14RUCB}. However due to the fixed confidence and `adversarial-sleeping' nature of the problem, we need to differently maintain pairwise confidence bounds per item (based on the availability sequence $\{S_t\}$), 
%, and also introduce the idea of itemwise `dominance-sets' 
which makes the resulting algorithm and its subsequent analysis significantly different than standard UCB based dueling bandit algorithms: Precisely given any $\delta > 0$, \algrucbp\, achieves a regret of $O\Big( \frac{ K^3\log (1/\delta) }{\Delta^2}\Big)$ with probability at least $1-\delta$ over any problem instance of \SPAA$(\P,\cS_T)$ (Sec.~\ref{sec:algo_rucbp}).  %Our regret analysis (Thm.~\ref{thm:ub_rucb}) proves near optimality of our proposed method. 
%Though the latter is theoretically suboptimal (by a factor $O(K/\Delta$), \algrucb~ performs very well in our simulations (Sec.~\ref{sec:expts}). 

\textbf{3.} In Sec.\,\ref{sec:algo_komiyama}, we design another computationally efficient algorithm, \algrmed\, (Alg.\,\ref{alg:rmed}), for \emph{`expected regret'} guarantee. 
Unlike the previous algorithm (\algrucbp), \algrmed\, uses empirical divergence (ED) based measures to filter out the `good' set of arms, inspired from the idea of RMED algorithm of \cite{Komiyama+15} for standard dueling bandits; however, due to sleeping nature of the items, it requires a different maintenance of `good' arms and the regret analysis of the {algorithm requires derivation of new results (as described in Sec.~\ref{sec:algo_komiyama})}.
The algorithm is shown to perform near optimally with an expected \emph{non-asymptotic} regret upper-bound of 
$%\label{eq:ub_rmed}
	\textstyle{ O\big( \min \big\{ KT^{2/3}, \sum_{j = 2}^{K} \frac{K \log T}{\Delta(j-1,j)}\big\}\big)}
$ (Thm.~\ref{thm:ub_rmed}). 
Note that for any problem instance with constant suboptimality gaps $\Delta(i,j) = \Delta$ for all $i<j$, regret bound of \algrmed~ is tight and matches the lower-bound ensuring the near optimality of \algrmed\, in the worst case.
Furthermore, a novelty of our finite time regret analysis lies in showing a cleaner tradeoff between regret vs. availability sequence $\cS_T$ which automatically adapts to the inherent `hardness' of the sequence of available subsets $\cS_T$, compared to existing sleeping bandits work for adversarial availabilities in the MAB setting \cite{kleinberg+10} which only gives a worst-case regret bound over all possible availability sequences (Rem.~\ref{rem:rmed}).	

\textbf{4.} Finally we corroborate our theoretical results with extensive empirical evaluations. (Sec.\,\ref{sec:expts}).
%\fi

%\vspace{3pt}
\textbf{Related Works.} 
The problem of regret minimization for stochastic multiarmed bandits (MAB) is extremely well studied in the online learning literature \cite{Auer+02,TS12,CsabaNotes18,Audibert+10,
Kalyanakrishnan+12}, where the learner gets to see a noisy draw of absolute reward feedback of an arm upon playing a single arm per round. 

A well motivated generalization of MAB framework is \emph{Sleeping Bandits} \cite{kanade09,neu14,kanade14,kale16}, much studied in the online learning community, where at any round the set of available actions could vary stochastically based on some unknown distributions over the decision space of $K$ items \cite{neu14,cortes+19} or adversarially \cite{kale16,kleinberg+10,kanade14}. 
Besides the reward model, the set of available actions could also vary stochastically or adversarially \cite{kanade09,neu14}. The problem is  NP-hard when both rewards and availabilities are adversarial \cite{kleinberg+10,kanade14,kale16}.
In case of stochastic reward and adversarial availabilities  \cite{kleinberg+10} proposed an UCB based no-regret algorithm, which was also shown to be provably optimal. 
The case of adversarial reward and stochastic availabilities has also been studied where the achievable regret lower bound is known be $\smash{\Omega(\sqrt{KT})}$ by the inefficient EXP$4$ algorithm \cite{kleinberg+10,kale16}. 
%However, the best known efficient algorithm for this setup is still $\tilde O((TK)^{2/3})$,\footnote{$\tilde O(\cdot)$ notation hides the logarithmic dependencies.} \cite{neu14}.

On the other hand over the last decade, the relative feedback variants of stochastic MAB problem has seen a widespread resurgence in the form of the Dueling Bandit problem, where, instead of getting noisy feedback of the reward of the chosen arm, the learner only gets to see a noisy feedback on the pairwise preference of two arms selected by the learner. The objective of the learner is to minimize the regret with respect to `best arm in the stochastic model. Several algorithms have been proposed to address this dueling bandits problem, 
for different notions of `best arms' or preference models \cite{Busa_mallows,Busa_pl,Zoghi+14RCS,Zoghi+14RUCB,Zoghi+15,Komiyama+15,DTS,CDB}, or even extending the pairwise preference to subsetwise preferences  \cite{Sui+17,Brost+16,SG18,SGwin18,Ren+18}.
However, surprisingly, unlike the `sleeping bandits generalization' of MAB, no parallel has been drawn for dueling bandits, which remains our main focus. % of this work. %despite the huge practical relevance of the sleeping-setup for preference based learning framework as explained earlier. To the best of our knowledge, we are the first to attempt the regret minimization of dueling bandits problem for stochastic preferences and adversarial availabilities. 

%!TEX root = main.tex

%* Write Lem. 3 proof.

%* Write PL properties
% \vspace{-8pt}
\section{Problem Formulation}
\label{sec:prob}
% \vspace{-5pt}
\textbf{Notations.} Decision space (or item/arm set) $[K]: = \{1,2,\ldots, K\}$. The available set of items at round $t$ is denoted by $S_t \subseteq [K]$.
%Also assume that $a_i \in [0,1]$ denotes the probability of item $i \in [n]$ being available at round $t$, we also assume availability of items are independent of each other.%, or in other words, for all item $i \in [n]$, $a_i \sim Ber()$
For any matrix $\M \in \R^{K \times K}$, we define $m_{ij} := M(i,j),~\forall i,j \in [K]$. We write $\smash{S_{\sm i} = S \sm \{i\}}$, for any $S \subseteq [K]$ and $i \in S$.
$\1(\cdot)$ denotes the indicator random variable which takes value $1$ if the predicate is true and $0$ otherwise and $\lesssim$ a rough inequality which holds up to universal constants. For any two items $x,y \in [K]$, we use the symbol $x \succ y$ to denote $x$ is preferred over $\y$. $\bSigma_{K}$ denotes the set of all permutations of the items in set $[K]$. The KL-divergence of two Bernoullis with biases $p$ and $q$ respectively is written $\smash{\kl(p, q) := p \log (\nicefrac{p}{q}) + (1-p) \log (\nicefrac{(1-p)}{(1-q)}})$. We assume $\frac{0}{0}:=0.5$ (in Alg.~\ref{alg:rucbp} and~\ref{alg:rmed}).

%\subsection{Setup and Objective}
%\label{sec:setup}
\noindent \textbf{Setup. } We consider the problem of stochastic $K$-armed dueling bandits with adversarial availabilities: At every iteration $t=1,\dots,T$, a set of available items (actions) $S_t \subseteq [K]$ is revealed, and the learner is asked to choose two items $x_t,y_t \in S_t$. Then, the learner receives a preference feedback $\smash{o_t = \1(x_t \succ y_t) \sim \text{Ber}(\P(x_t,y_t))}$, where $\smash{\P \in [0,1]^{K \times K}}$ is an underlying pairwise preference matrix, unknown to the learner.  The setting is described in Figure~\ref{fig:setting}.
We assume that $\P$ respects a \emph{`total ordering'}, say $\bsigma^* \in \bSigma_{K}$. Without loss of generality, we set $\bsigma^* = (1,2,\ldots, K)$ thoughout the paper. This implies $\smash{\P(i, j) \geq 0.5}$ for $i \leq j$. One possible pairwise probability model which respects \emph{`total ordering'} is Plackett-Luce \cite{Az+12}, where it is assumed that the $K$ items are associated to positive score parameters $\theta_1, \ldots, \theta_K$, and $\P(i,j) = {\theta_i}/({\theta_i + \theta_j})$ for all $i,j \in [K]$. In fact any well random utility (RUM) based preference model would have the above property, like \cite{Az+12,Az+13}. Note also that our assumption corresponds to assuming the existence of a Condorcet winner for every subset $S_t \subseteq [K]$.
%With slight abuse of notation we will henceforth denote $\sigma^*(S) = \pi^*(S)$.

\noindent \textbf{Objective. }
The objective of the learner is to minimize his regret over $T$ rounds with respect to the best policy in the policy class $\Pi = \{\pi: 2^{K} \mapsto [K] \mid \forall t \in [T],~ \pi(S_t) \in S_t\}$, i.e. any $\pi \in \Pi$ is such that for any $t \in [T]$, $\pi(S_t) \in S_t$. More formally we define the regret as follows:
\vspace*{-2pt}
\begin{align}
\label{eq:reg}
R_{T} = \max_{\pi \in \Pi} \sum_{t = 1}^T \frac{\P(\pi(S_t),x_t)+\P(\pi(S_t),y_t)-1}{2} \,.
\end{align}
We analyze both \emph{fixed-confidence} and \emph{expected} regret guarantees in this paper respectively in Sec.~\ref{sec:algo_rucbp} (see Thm.~\ref{thm:ub_rucb}) and  Sec.~\ref{sec:algo_komiyama} (see Thm.~\ref{thm:ub_rmed}).
It is easy to note that under our preference modelling assumptions, the best policy, say $\pi^*$, turns out to be $\pi^*(S) = \min\{S\}$ for any $S \subseteq [K]$. We henceforth denote by $i_t^* = \pi^*(S_t)$.
We define the above problem to be \dspaa \, over the stochastic preference matrix $\P \in [0,1]^{K \times K}$ and the sequence of available subsets $\cS_T = \{S_1,\ldots, S_T\}$, or in short \SPAA$(\P,\cS_T)$.
For ease of notation we respectively define the gaps  and the non-zeros gaps as $\Delta(i,j) := \P(i,j)-{1}/{2}$, and \vspace*{-2pt}
\begin{align}
\label{eq:gap}
 \Delta(i,j)_+ := \left\{
		\begin{array}{ll}
		\Delta(i,j) & \text{if } \Delta(i,j) \neq 0 \\
		+ \infty  & \text{if }  \Delta(i,j) = 0
		\end{array} \right.
\end{align}
The regret thus can be rewritten as $\smash{R_{T} := \sum_{t = 1}^{T}r_t}$, where $\smash{r_t := (\Delta(i_t^*,x_t) + \Delta(i_t^*,y_t))/2}$ denotes the instantaneous regret. We also denote by $\smash{n_{ij}(t):= \sum_{\tau = 1}^{t}\1\big(\{x_t,y_t\} = \{i,j\}\big)}$ the number of times the pair $(i,j)$ is played until time $t$ and by $w_{ij}(t)$ the number of times $i$ beats $j$ in $t$ rounds.
% 

%!TEX root = main.tex
\def \mnl{{MNL($n,\btheta$)}}
\def \nr{{\it No-regret}}
\vspace{-2pt}
\section{Lower Bound}
\label{sec:lb}
\vspace{-2pt}
We first derive a worst case regret lower bound over all possible sequences of $\cS_T$. 
The proof idea essentially lies in constructing \emph{hard enough} availability sequences $\cS_T$, where no learner can escape learning the preferences of every distinct pair of items $(i,j)$. This leads to a potential lower bound of $\Omega\big( K^2 \log(T)/\Delta\big)$. For this section we denote $\P$ by $\P_K$ to make the dependency on $K$ more precise.

\begin{restatable}[Lower Bound for \SPAA$(\P_K,\cS_T)$]{thm}{lb}
\label{thm:lb}
For any \nr\, learning algorithm $\cA$, there exists a problem instance {\SPAA}$(\P_K,\cS_T)$ with $T \ge K^4$, such that its expected regret is lower-bounded as: 
\vspace{-6pt}
\begin{align*}
\vspace{-10pt}
\E[R_T(\cA)] \ge \Omega\bigg( \sum_{i = 1}^{K-1}\sum_{j = i+1}^K \frac{\log T}{\Delta(i,j)_+} \bigg) \,.
\end{align*}
\vspace{-13pt}
\end{restatable}

The \nr\, {\it learning algorithm } refers to the class of `consistent algorithms' which do not pull any suboptimal pair more than $O(T^\alpha)$, $\alpha \in [0,1]$ (see Def.~\ref{def:con}, Appendix~\ref{app:lb})).

\noindent \textbf{Proof (sketch)}
The main argument lies behind the fact that in the worst case the adversary can force the algorithm to learn the preference of every distinct pair $(i,j)$ as for the `worst-case' sequences $\cS_T$, a knowledge of the already `learnt' pairwise preferences would not disclose any information on the remaining pairs; e.g. assuming $\bsigma^* = (1,2,\ldots, K)$, revealing the available subsets in the following sequence $(1,2), (1,3), \ldots (1,K), (2,3), (2,4), \ldots (K-1,K)$ would force the learner to explore (learn the preferences) all $\smash{K \choose 2}$ distinct pairs. 
The remaining proof establishes this formally, towards which we first show a $\Omega({\ln (T)}/{\Delta(1,2)})$ regret lower bound for a \SPAA \, instance with just two items (i.e. $K=2$) as shown in Lem.~\ref{lem:lb}. The lower bound for any general $K$ can now be derived applying the above bound on independent $\smash{K \choose 2}$ subintervals, with the availability sequence $(1,2), (1,3), \ldots (1,K), (2,3), (2,4), \ldots (K-1,K)$. 
The full proof is given in Appendix~\ref{app:lb}.
$\hfill \square$

\begin{restatable}[Lower Bound of \SPAA$(\P_K,\cS_T)$ for $2$ items]{lem}{lemlb}
\label{lem:lb}
For any \nr\, learning algorithm $\cA$, there exists a problem instance \SPAA$(\P_2,\cS_T)$ such that the expected regret incurred by $\cA$ on that can be lower bounded as:
$
\E[R_T(\cA)] \ge \Delta^{-1}\log (T)
$, $\Delta$ being the `preference-gap' between the two items (i.e. $\Delta = \P_{12}-\nicefrac{1}{2}$, assuming $P_{12}>\nicefrac{1}{2}$ or equivalently $\Delta > 0$).
\end{restatable}

\begin{rem}[Implication of the lower bound]
\label{rem:lb}
\emph{
The above result indicates that in the preference based learning setup, the fundamental problem complexity lies in distinguishing every pair of items $1\le i<j\le K$. If the learner fails to learn the preference of any pair $(i,j)$, the adversary can make the learner suffer $O(T)$ regret by setting $S_t = \{i,j\}$ henceforth at all round. It is worth noting that in the `no sleeping' case both dueling bandits and MAB are known to have the same fundamental performance limit of $\Omega(K \log (T)/\Delta)$ (assuming $\P$ respects a \emph{condorcet winner} \cite{Auer+02,Komiyama+15}). Thus Thm.~\ref{thm:lb} shows that the `sleeping-aspect' of dueling bandits makes the problem $K$-times harder than `sleeping-MAB' for which the regret lower bound is known to be only $\Omega\big(\sum_{i = 1}^{K}{\log (T)}/{\Delta(i,i+1)}\big)$~\cite{kleinberg+10}.}
%
%We also derive an matching $\Omega\Big( \sum_{i = 1}^{K-1}\sum_{j = i+1}^K \frac{\log T}{\Delta(i,j)} \Big)$ regret lower bound guarantee (Thm.~\ref{thm:lb}, Sec.\,\ref{sec:lb}) for the problem which establishes the near optimality of our proposed algorithms. \red{change} 

% single identifying the consecutive items... \red{rephrase}.
\end{rem}

%!TEX root = main.tex

\section{\algrucbp: A Fixed-Confidence Algorithm}
\label{sec:algo_rucbp}
% \vspace{-5pt}
In this section, we design an efficient algorithm for the \SPAA$(K,T)$ problem with instance-dependent regret guarantee. 
% $O\big( \sum_{i = 1}^{K-1}\sum_{j = i+1}^K \frac{\log T}{\Delta(i,j)} \big)$, 
%that matches the lower bound of Thm.~\ref{thm:lb}.

% \vspace{-2pt}
\noindent \textbf{Main ideas. } 
Our algorithm, described in Alg.~\ref{alg:rucbp}, depends on an hyper-parameter $\alpha > 0.5$ and a confidence parameter $\delta >0$. It maintains, for each item $k \in [K]$, its own record of empirical pairwise estimates of the duels, $(i,j) \in [K]\times[K]$ and their respective upper confidence bounds defined as: 
\[
\textstyle{\hp_{ij}(t):=\frac{w_{ij}(t)}{n_{ij}(t)}   \qquad \text{and} \qquad u_{ij}(t): = \hp_{ij}(t)+ c_{ij}(t), \text{ with} \quad c_{ij}(t) := \sqrt{\frac{\alpha \log a_{ij}(t)}{n_{ij}(t)}} \,,}
\]
where  $w_{ij}(t)$ denotes the total number of times item $i$ beats $j$ up to round $t$, $n_{ij}(t):= w_{ij}(t) + w_{ji}(t)$, and for all $i,j \in [K]$ and $t \in [T]$
\[
	\textstyle{a_{ij}(t) := \max\{C(K,\delta),n_{ij}(t)\} \qquad \text{and} \qquad \smash{ C(K,\delta) := \Big(\frac{(4\alpha-1) K^2}{(2\alpha-1)\delta}\Big)^{\frac{1}{2\alpha-1}}}\,.}
\]
A key observation is that our careful choice of the confidence bounds $c_{ij}(t)$ ensures that with high probability $p_{ij}(t) \in [\hp_{ij}(t)-c_{ij}(t),\hp_{ij}+c_{ij}(t)]$ for any duel $i,j \in [K]$ and any $t \in [T]$ (Lem.~\ref{lem:conf_cdels}). 
Now at any round $t\geq 1$, the algorithm first computes a set of potential winners of $S_t$ as 
$
\cC_t = \{k \in S_t \mid |\cC_k(t)| = \max_{j \in S_t}|\cC_j(t)| \}, 
$ 
where $\cC_k(t) := \{j \in S_t \mid u_{kj}(t) > \frac{1}{2} \}$ denotes the set of items that item $k$ dominates (optimistically). At each round, we play a random item from the set potential winners $\cC_t$ as the left arm $x_t$. 
Finally the right-arm $y_t$ is chosen to be the most competitive opponent of $x_t$ as $y_t \leftarrow \arg\max_{i \in \cC_t} u_{ji}(t)$ from the potential winners. Our arm selection strategy ensures that eventually for all $t$, algorithm plays the optimal pair $(i_t^*,i_t^*)$ frequently enough as desired.

\begin{algorithm}[t]
   \caption{\textbf{\algrucbp}}
   \label{alg:rucbp}
\begin{algorithmic}[1]	
\STATE {\bfseries input:} Arm set: $[K]$, parameters $\alpha > 0.5$, Confidence parameter $\delta \in [0,1)$
\STATE {\bfseries init:} $w_{ij}(1) \leftarrow 0$, $\cD_i(1) \leftarrow \emptyset, \, \forall i,j \in [K]$. 
\STATE \textbf{define: } $n_{ij}(t) := w_{ij}(t) + w_{ji}(t), \, \forall t \in [T]$ 
%\STATE Play any random pair $(x_t,y_t) \in S_t \times S_t$ for the initial $t \in [t_0]$ rounds. 
%\STATE For all $t \in [t_0]$: Set $a_{ij}(t) \leftarrow t$. Update $w_{ij}(t+1) \leftarrow w_{ij}(t)$, $\forall i,j \in [K]$; If $i$ beats $j$, then set $w_{ij}(t+1) \leftarrow w_{ij}(t)+1, ~\forall k \in S_t$, $\forall t \in [t_0]$.
%
\FOR{$t = 1, 2, \ldots, T$}
\STATE Receive $S_t \subseteq [K]$	
\STATE ${\displaystyle \hp_{ij}(t)= \frac{w_{ij}(t)}{n_{ij}(t)}, \, c_{ij}(t) \leftarrow \sqrt{\frac{\alpha \log a_{ij}(t)}{n_{ij}(t)}}, ~\forall i,j \in S_t}$,  (assume $\frac{x}{0}:=0.5, ~\forall x \in \R)$ %~~\big(let $\nicefrac{0}{0}:=0.5$\big)
\STATE $u_{ij}(t) \leftarrow \hp_{ij}(t) + c_{ij}(t), ~u_{ii}(t) \leftarrow \nicefrac{1}{2}, \, \forall i,j \in S_t$  \hfill $\triangleright$ {\color{black} UCB of empirical preferences}
\FOR {$k \in S_t$}
\STATE $\cC_k(t) := \{j \in S_t \mid u_{kj}(t) > \nicefrac{1}{2} \}  \hfill \triangleright$ {\color{black} Potential losers to $k$}

\ENDFOR

\STATE $\cC_t = \{i \in S_t \mid |\cC_i(t)| = \max_{j \in S_t}|\cC_j(t)| \}
\hfill \triangleright$ {\color{black} Potential best items}
\STATE Select a random $x_t$ from $\cC_t$. Choose $y_t \leftarrow \arg\max_{i \in \cC_t} u_{i x_t}(t)$ %\max_{k \in \cC_t}, \red{could be better $[\min_{j \in S_t}u^{(j)}_{i x_t}(t)]$?} 	
%
%\STATE Update $a_{i j}(t+1) \leftarrow 
%\begin{cases}
%a_{i j}(t) + 1, ~\forall k \in \{x_t,y_t\},\, i \in %\{x_t,y_t\}, j \in S_t\sm\{x_t\}\\
%a_{i j}(t), ~\text{ otherwise}
%\end{cases}$ 
%\STATE Play $(x_t,y_t)$ 
\STATE Play $(x_t,y_t)$. Receive preference $o_t$ %$\sim \text{Ber}(\P(x_t,y_t))$
\STATE Update: $\forall i,j \in [K]$, $w_{x_t y_t}(t+1) \leftarrow w_{x_t y_t}(t) + o_t$, $w_{y_t x_t}(t+1) \leftarrow w_{y_t x_t}(t) + (1-o_t)$, $a_{ij}(t+1) \leftarrow \max\{n_{ij}(t),C(K,\delta)\}$, $\forall i,j \in [K]$ 
\ENDFOR 
%\STATE Return $\bsigma'$
\end{algorithmic}
\end{algorithm}  
% \vspace{-20pt}

%
% The complete description of \algrucbp\, is given in Alg.~\ref{alg:rucbp}.

% \subsection{Regret Analysis. }
% \vspace{-7pt}
\label{sec:rucbp_analysis}
%It would be convenient to introduce some technical notations before stating the regret guarantee of Alg.~\ref{alg:rucbp} formally: 

%\vspace{-4pt}
% \noindent \textbf{Notations. } 
% For any $i \in[K]$ we write 
% $\Delta^{(k)}_{ij} := \min\{ \Delta(k,i)_+ ,\Delta(k,j)_+ \}$, 
 % $\smash{\dk_{\max} := \max_{i \ge k} \Delta(k,i)}$ and $\smash{\Delta_{\min}:= \min_{i} \Delta(i,i+1)_+}$. 
% We recall that $\alpha > 1/2$ is a parameter of the algorithm. 

%  and we define
% \[
% 	\smash{
% 		\textstyle{N_{ij}^{(k)} := 4\alpha \min\{ \Delta(k,i)_+ ,\Delta(k,j)_+ \}^{-2} \quad \text{and} \quad \Nk := \sum_{k \le i < j \le K}\Nk_{ij} \,.}
% 		}
% \]  
% $C(K,\delta) := \big({(4\alpha-1) K^2}/((2\alpha-1)\delta)\big)^{\frac{1}{2\alpha-1}}$.

\begin{restatable}[Fixed-confidence regret analysis: \algrucbp]{thm}{ubrucbp}
\label{thm:ub_rucb}
Given any $\delta > 0$ and $\alpha\ge1$, with probability at least $1-\delta$, the regret incurred by \algrucbp ~(Alg.~\ref{alg:rucbp})  
is upper-bounded as: 
\vspace*{-3pt}
\begin{align*}
R_T  & \leq 2 \sum_{i=1}^{K-1} \sum_{j=i+1}^K M_{ij} \log \big(2 C(K,\delta) M_{ij}\big)  
\end{align*}
where  
\[
	C(K,\delta) := \bigg(\frac{(4\alpha-1) K^2}{(2\alpha-1)\delta}\bigg)^{\frac{1}{2\alpha-1}}\ \text{and}\ M_{ij} = \sum_{k=1}^i \frac{4\alpha}{\min\big\{\Delta(k,i)_+ ,\Delta(k,j)_+ \big\}^2} \,.
\]
\end{restatable}

The complete proof with a precise dependencies on the model parameters is deferred to Appendix~\ref{app:rucbp}.

\begin{rem} The dependency on $\Delta = \min_{i,j} \Delta_+(i,j)$ does not match the lower-bound of Thm.~\ref{thm:lb}, which is of order  $O(\log (T)/ \Delta)$. Instead, Thm.~\ref{thm:ub_rucb} proves $O(\log(1/\delta)/\Delta^2)$. Yet, the bounds are not directly comparable because the lower-bound is on the expected regret while the upper-bound considers fixed-confidence $\delta$ and is hence independent of $T$. All existing dueling bandit algorithms, that minimize the expected regret, suffer an additional constant term of order $O(1/\Delta^2)$ --see for instance \cite{Zoghi+14RUCB,Komiyama+15}. Achieving an order $O(1/\Delta)$ dependendence is an interesting question for future work. 
\end{rem}

%O\big(KC(K,\delta) + \sum_{k = 1}^{K-1}\Nk\log \Nk\big)   + \sum_{k = 1}^{K-1}\sum_{i = k+1}^{K}\frac{ 2\alpha \log T }{\Delta(k,i)_+}  \lesssim \Big(\frac{K^3}{\delta}\Big)^{\frac{1}{2\alpha-1}} + \frac{K^3 \log(\frac{K^2}{\Delta_{\min}})}{{\Delta_{\min}}^2} +  \sum_{k = 1}^{K-1}\sum_{i = k+1}^{K}\frac{\alpha \log T }{\Delta(k,i)_+}.

\iffalse%%%%%%%%%%%%
{\color{red}
\begin{rem}[Optimality of \algrucbp{}]
	\label{rem:vsrucb}
% Note that the first term in the regret upper bound is simply $O(K^3/\delta)^{\frac{1}{2\alpha -1}}$ ($\alpha > 0.5$ being the choice of the algorithm), and the second term simplifies to $O\big(\frac{K^3}{{\Delta_{\min}}^2}\big)$---both of which are independent of $T$. It is only the third term that depends on $T$ as $O\Big( \sum_{k = 1}^{K-1}\sum_{i = k+1}^{K}\frac{\log T }{\Delta(k,i)_+}  \Big)$ and thus 
For a given $\delta \in [0,1)$, the regret bound of Alg.~\ref{alg:rucbp} is \emph{instance-wise optimal} as follows from the lower bound (Thm.~\ref{thm:lb}). This actually improves the state of the art $\smash{O(K/\Delta^2)}$ regret bound of the UCB based dueling algorithms (e.g. RUCB \cite{Zoghi+14RUCB}, which is known to be $O(K\log T/ \Delta^2)$, while our \algrucbp{} regret bound can shown to be only $\smash{O(K\log T/ \Delta)}$ for the standard condorcet-dueling bandit objective (without sleeping). Thus Alg.~\ref{alg:rucbp} also performs optimally for the standard dueling setting matching the known lower bound of $\Omega(K \log T / \Delta)$ \cite{Komiyama+15}. \vspace*{-10pt}
\end{rem}}
\fi%%%%%%%%%%%%%%%%%%%%%%%%%%

%\begin{rem}[Expected regret bound]
%\emph{Following the technique of \cite{Zoghi+14RUCB} (Thm.~5), one can derive an expected regret bound for the algorithm from the high probability regret guarantee of Thm.~\ref{alg:rucbp}.} \vspace{-5pt}
%\end{rem}

\begin{proof}[Proof sketch of Thm.~\ref{thm:ub_rucb}]
%We start with describing an intuitive proof sketch. 
The key steps lie in proving the following four lemmas.
The first lemma follows along the line of Lem.~1 of RUCB algorithm~\cite{Zoghi+14RUCB}and shows that all the pairwise estimates are contained within their respective confidence intervals with high probability. 

\begin{restatable}[]{lem}{lemzoghi}
\label{lem:conf_cdels}
Let $\alpha >0.5$ and $\delta >0$. Then, for any $i,j \in [K]$, with probability at least $1-\delta/K^2$, %
\[%\begin{equation}
	     \hp_{ij}(t)-c_{ij}(t) \leq p_{ij} \leq  u_{ij}(t) := \hp_{ij}(t)+c_{ij}(t),  \qquad \forall t \in [T] \,.
%\label{eq:lem5}
\]%\end{equation} 
\end{restatable}
The lemma below shows that once the algorithm can not play any suboptimal pair `too many times'. 

\begin{restatable}[]{lem}{rucbnijs}
\label{lem:rucb_nijs}
	Let $\alpha > 0.5$. Under the notations and the high-probability event of Lem.~\ref{lem:conf_cdels}, for all $i,j,k \in [K]$ such that $\{i,j\} \neq \{k,k\}$, and for any $\tau \geq 1$
	\begin{equation*}
		\sum_{t = 1}^{\tau} \1(i_t^* = k)\1\big(\{x_t,y_t\} = \{i,j\}\big) 
		\le \frac{4\alpha \log a_{i,j}(\tau)}{\min\big\{\Delta(k,i)_+ ,\Delta(k,j)_+ \big\}^2} \,,
	\end{equation*}
where recall $a_{ij}(\tau) = \max\big( C(K,\delta),n_{ij}(\tau)\big)$.
\end{restatable}

 With probability of at least $1-\delta$, the event of Lem.~\ref{lem:conf_cdels} holds and thus so do the ones of Lem.~\ref{lem:rucb_nijs}. The regret of Alg.~\ref{alg:rucbp} then follows from applying the above lemmas with the following careful decomposition of the regret: 
{
\begin{align*}
R_T & \le \sum_{i=1}^{K-1} \sum_{j=i+1}^K \sum_{t=1}^T \sum_{k=1}^i \1(i_t^* = k)  \1\big(\{x_t,y_t\} = \{i,j\}\big) = \sum_{i=1}^{K-1} \sum_{j=i+1}^K n_{ij}(T) 
\end{align*}
and the proof is concluded by using Lemma~\ref{lem:rucb_nijs} to upper-bound $n_{ij}(T)$. The complete proof given in Appendix~\ref{app:rucbp}.
}
%where $(a)$ follows from Lem.~\ref{lem:hat_cdel}.
\end{proof}

%%%%%%%%%%%%%%%%%%%%%%%%%%%%%%%%%%%

%\input{algo_rucbp_neurips21version.tex}

%!TEX root = main.tex
\section{\algrmed: An Expected Regret Algorithm}
\label{sec:algo_komiyama}

% We thus provided Alg.~\ref{alg:rucbp}, which achieves an optimal regret for \SPAA$(\P,\cS_T)$. However, its space and runtime complexities are of order $O(K^2$ (see Rem.~\ref{rem:rucb+}) which may be prohibitive for large K. 
In this section, we propose another computationally efficient algorithm, \algrmed\, (Alg.~\ref{alg:rmed}), which achieves near-optimal expected-regret for \SPAA\, problem, and also performs competitively against \algrucbp\, empirically (see Sec.~\ref{sec:expts}). 
Furthermore, a novelty of our finite time regret analysis of \algrmed\, lies in showing a cleaner trade-off between regret vs availability sequence $\cS_T$ which automatically adapts to the inherent `hardness' of the sequence of available subsets $\cS_T$, unlike the previous attempts made in standard sleeping bandits for adversarial availabilities \cite{kleinberg+10} (Rem.~\ref{rem:rmed}).	

 % which \emph{asymptotically} gives an improved guarantee over Alg.~\ref{alg:rucb}.\red{Correct?/Rephrase}. We also derive a regret bound with explicit dependencies on the availability sequence $\cS_T$ unlike the existing attempts for sleeping bandits problems in multiarmed bandit setting \cite{kleinberg+10}, hence for \emph{easier} $\cS_T$ our regret bounds yields a much small guarantee (see Rem.~\ref{rem:ub_gap}). 

\begin{algorithm}[!t]
   \caption{\textbf{\algrmed}}
   \label{alg:rmed}
\begin{algorithmic}[1]	
   \STATE {\bfseries input:} Arm set: $[K]$, exploration parameter $t_0 > 0$, parameter $\alpha > 0$
   %\STATE {\bfseries init:} $w_{ij}(t) \leftarrow 0,~n_{ij}(t) \leftarrow 0 ~\forall i,j \in [K]$
   %\STATE {\bfseries init:} $\cG_{curr}, \cG_{rem} \leftarrow [K], ~\cG_{next} \leftarrow \emptyset$.
   %\STATE Play each distinct arm-pair $(x,y) \in [K]$ as $(x_t,y_t)$ for $t_0$ rounds, and set $T_0 = \frac{(K-1)K t_0}{2}$
   %\STATE For $t = T_0$, set $w_{ij}(t): = \sum_{\tau=1}^{t}\1(i \succ j)$,  $n_{ij}(t): = w_{ij}(t) + w_{ij}(t), ~\forall i,j \in [K]$
   \FOR{$t = 1, 2, \ldots, T$}
	\STATE $\hp_{ij}(t) := \frac{w_{ij}(t)}{n_{ij}(t)}, ~\hp(i,i) \leftarrow \nicefrac{1}{2},~\forall i,j \in [K]$ (assume $\frac{x}{0}:=0.5, ~\forall x \in \R)$ ~~%\big(we assume $\frac{0}{0}:=0.5$\big)
	\STATE Receive $S_t \subseteq [K]$
    \IF {$|S_t|\ge 2$ and $\exists i,j \in S_t$ s.t. $n_{ij}(t) < t_0, ~i \neq j$}
    \STATE Set $x_t \leftarrow i$, $y_t \leftarrow j$ \hfill  $\triangleright$ {\color{black} Exploration rounds}
	\ELSE
	\STATE $\hcB_i(t):= \{j \in [K]\sm\{i\} \mid \hp_{ij}(t) \le \nicefrac{1}{2} \}\cap S_t, ~\forall i \in [K]$ \hfill $\triangleright$ {\color{black} Empirical winners over $i$}
	\STATE $\cI_i(t) := \sum_{j \in \hcB_i(t)}n_{ij}(t)\kl(\hp_{ij}(t) , \nicefrac{1}{2} )$, $\forall i \in [K]$ and $\hi^*_t \leftarrow \arg\min_{i \in S_t}\cI_i(t)$
    \STATE $\cC_t:= \{i \in S_t \mid \cI_i(t) - \cI_{\hi^*_t}(t) \le \alpha \log t\}$ \hfill  $\triangleright$ {\color{black} Potential good arms}
	\STATE Select any $x_t$ from $\cC_t$ uniformly at random
    \STATE \textbf{if} \big(\, $\hi^*_t \in \hcB_{x_t}(t)$ or $\hcB_{x_t}(t) = \emptyset$\big): set $y_t \leftarrow \hi^*_t$, \textbf{else}: $y_t \leftarrow \arg\max_{i \in S_t\sm\{x_t\}} \hp_{i x_t}(t)$
	\ENDIF
	\STATE Play $(x_t,y_t)$ Receive preference feedback $o_t$ %$\sim \text{Ber}(\P(x_t,y_t))$. 
	%\STATE Update: $w_{x_t,y_t}(t+1) \leftarrow w_{x_t,y_t}(t) + o_t$, and $w_{x_t,y_t}(t+1) \leftarrow w_{x_t,y_t}(t) + (1-o_t)$. 
    %\STATE Set $n_{x_t,y_t}(t+1) = n_{x_t,y_t}(t+1) = n_{x_t,y_t}(t)+1$
   \ENDFOR  
   %\STATE Return $\bsigma'$
\end{algorithmic} 
\end{algorithm} %\vspace{-10pt}

%\vspace{5pt}
\noindent \textbf{Main ideas. }
%\textbf{Notations. } 
%The main idea of algrmed\, is inspired from the \emph{mean-empirical divergence} based RMED algorithm of \cite{Komiyama+15} for standard dueling bandits. 
We again use the notations $w_{ij}(t), n_{ij}(t)$ as used for \algrucbp \,(Alg.~\ref{alg:rucbp}), with the same initializations. %, $\forall i,j \in [K]$, $S \subseteq [K], \, |S| \le 2$.
Same as \algrucbp, this algorithm also maintains the empirical pairwise preferences $\hp_{ij}(t)$ for each item pair $i,j \in [K]$. 
However, unlike the earlier case here we need to ensure an initial $t_0$ rounds of exploration ($t_0 = 1$ in the theorem) for every distinct pairs $(i,j)$, and instead of maintaining pairwise UCBs, in this case the set of `good-items' is defined in terms of \emph{empirical divergences} for all $i \in S_t$ \vspace*{-1pt}
\begin{align*}
  \cI_i(t) & := \sum_{j \in \hcB_i(t)}n_{ij}(t)\kl \big(\hp_{ij}(t) , \nicefrac{1}{2} \big), ~\, 
  \hcB_i(t):= \Big\{j \in [K]\sm\{i\} \mid \hp_{ij}(t) \le \nicefrac{1}{2} \Big\}\cap S_t
\end{align*}
denotes the empirical winners of item $i$ in set $S_t$. 
Now intuitively since $\exp(-\cI_i(t))$ can be interpreted as the likelihood of $i$ being the \emph{best-item} of $S_t$, we denote by $\smash{\hi^*_t \leftarrow \arg\min_{i \in S_t}\cI_i(t)}$  the \emph{empirical-best} item of round $t$ and define the set of `near-best' items 
$
  \smash{\textstyle{  \cC_t:= \big\{i \in S_t \mid \cI_i(t) - \cI_{\hi^*_t}(t) \le \alpha \log t \big\},}}
$
whose likelihood is close enough to that of $\smash{\hi^*_t}$. Finally the algorithm selects an arm pair $(x_t,y_t)$ such that $x_t$ is a potential candidate of good arm (which ensures the required exploration) and $y_t$ being the strongest challenger of $x_t$ w.r.t the empirical preferences. 
The algorithm is given in Alg.~\ref{alg:rmed}.

% \noindent \textbf{Notations.} 
% Let us also define: 
% $
% i_\epsilon(j):= \arg\min\{i \mid i \le j, \Delta(i,j) \le \epsilon\}, \text{and } j_\epsilon(i):= \arg\max\{j \mid j \ge j, \Delta(i,j) \le \epsilon\}$.

%Let $a_{ij}(T) = \sum_{t = 1}^T \1(\{i,j\} \subseteq \S_t)$ being the number of times $i,j$ appeared together in a set in $T$ rounds. 

% \iffalse% % % % % % % % % % % % % % % % % % % % %
% \red{(Not required)
% Now for any item $j \in [n]\sm\{1\}$, let us define $S_{(j)}:=\{i < j \mid a_{ij}(T) \ge N_{ij}(\delta)\}$ to be the set of items (superior to $j$), such that $i$ and $j$ appeared together in a set for greater than $N_{ij}(\delta)$ times. Thus $S^c_{(j)} := [j-1]\sm S_{(j)}$. Also for any $i \in S_{(j)} \cup \{j\}$, we also denote by $S_{(j)}(i^-):= \max\{k \in S_{(j)} \mid k < i\}$, $S_{(j)}(i^+):= \min\{k \in S_{(j)} \mid k > i\}$, $\underbar i(S_{(j)}) = \min\{i \in S_{(j)}\}$, $\bar i(S_{(j)}) = \max\{i \in S_{(j)}\}$, and $S_{(j)}(\underbar i(S_{(j)})^-)=0$, $S_{(j)}(\bar i(S_{(j)})^+)=K+1$.} 
% \fi % % % % % % % % % % % % % % % % % % % % %

%\noindent 
%We show the following regret upper-bound for \algrmed. 

%\textbf{Notations.} 

% \red{
% Aadirupa: we have space, so may be you can add the key lemmas.. \\
% Pierre: I prefer not. I find them too technical and they might lead to the wrong conclusion that the analysis is exactly the same as the one of Komiyama.  + We do not have that much space since page limit is 8p.}
\begin{restatable}[Expected regret analysis \algrmed]{thm}{ubrmed}
\label{thm:ub_rmed}
Let $t_0=1$ and $\alpha = 4K$. Then as $T \to \infty$, the expected regret incurred by \algrmed ~(Alg.~\ref{alg:rmed}) 
% over any problem instance \SPAA$(\P,\cS_T)$ 
can be upper bounded as: For all $\epsilon_2,\dots,\epsilon_K \geq 0$
{
\begin{align*}
\E \big[R_T\big] 
     & \lesssim K^2  + \hspace*{-.3cm}  \sum_{1\leq i <j \leq K} \hspace*{-.2cm} \bigg( \frac{K \1_{\{\Delta(i,j)>\epsilon_j\}}}{\Delta(i,j)^2 } +   n_{ij}(T) \min\{\epsilon_j, \Delta(i,j)\}
    \bigg) +  \sum_{j = 2}^{K} \frac{K \log T} {\max\big\{ \epsilon_j, \Delta(j-1,j)_+ \big\}} \\
    & \leq O\bigg( \min \bigg\{ \sum_{j = 2}^{K} \frac{K \log T}{\Delta(j-1,j)_+}, \ K T^{2/3} \bigg\}  \bigg).  
\end{align*}
}
\end{restatable}
The proof is deferred to Appendix~\ref{app:rmed}. Although it borrows some high-level ideas from~\cite{kleinberg+10} for sleeping bandits and from \cite{Komiyama+15} for RMED in standard dueling bandits, our analysis needed new ingredients in order to obtain $O(K^2 (\log T)/\Delta)$. This is especially the case for the proofs of the technical Lemmas~\ref{lem:goodevent} and~\ref{lem:ni} which significantly differ from ``corresponding'' technical lemmas of \cite{Komiyama+15}. Specifically, both regret bounds of RMED and ours need to control the length of an initial exploration $t_0$ after which pairwise preferences are well estimated by $\hat p_{ij}(t)$. This is done respectively by our Lemma~\ref{lem:goodevent} and Lemma~5 of \cite{Komiyama+15}. Yet, RMED's original analysis is based on a union bound over all possible subsets $S \subset \{1,\dots,K\}$ of items (see Equation~(19) in~\cite{Komiyama+15}), whose number is exponential in $K$.  Despite our efforts, we could not follow the proof of Lemma~5 of \cite{Komiyama+15}, which to the best of our understanding, should yield to an exploration $t_0$ exponentially large in $K$ contrary to $O(1)$ claimed in \cite{Komiyama+15}. Instead, in our proof of Lem.~\ref{lem:goodevent}, we carefully apply concentration inequalities to run union bounds over the items directly instead of sets of items. 

% \red{@Pierre: Can you write 2-3 lines how the analysis is novel compared to RMED, I think it would be good to pin-point it to stop the reviewers complaining about novelty.}

%
The upper-bound of Thm.~\ref{thm:ub_rmed} is close to optimal. It suffers at most a suboptimal factor $K$ and exactly matches the lower-bound for some problems. The distribution-free upper-bound of order $O(T^{2/3})$ matches obtainable standard dueling bandit problems \cite{Komiyama+15,Zoghi+14RUCB}, since the later algorithms also suffer constant terms of order $\Delta^{-2}$. Yet, it is unclear whether it is optimal or if $\smash{O(\sqrt{T})}$ can be obtained. 
% \red{@Aadirupa: can you check that? You know better than me dueling bandit litterature.} 

\begin{rem}[Sequence $\cS_T$ adaptivity of Alg.~\ref{alg:rmed}]
\label{rem:rmed}
\emph{It is worth pointing out that the regret bound of Thm.~\ref{thm:ub_rmed} is finite time and automatically adapts to the sequence of available sets $\cS_T$. In the worst-case, the complexity lies in identifying for all items $j$ the gap with the earlier item $j-1$. Yet, our regret-bound, which holds for any $\epsilon_j \geq 0$, will automatically perform a trade-off for each $j$ between the gap $\smash{\Delta(j-1,j)_+^{-1}}$ and $\smash{\epsilon_j \sum_{i=1}^{j-1} n_{ij}(T)}$ the number of times $j$ is played together with a better item $i<j$. In particular, $\smash{\sum_{i=1}^{j-1} n_{ij}(T)}$ can be small if $j$ is rarely available in $S_t$ while not optimal. Notably, this adaptivity to $\cS_T$ item per item improves the regret guarantee of Thm. $(10)$, \cite{kleinberg+10}, which also addresses the problem of sleeping bandits with `adversarial availabilities' but for the stochastic multi-armed bandit setup and only provides worst-case guarantees over all $\cS_T$ and a trade-off $\epsilon$ independent of $j$.}
\end{rem}

\iffalse%%%%%%%%%%%%%%%
\begin{rem}[Space and runtime complexity]
	\label{rem:rucb+} 
\emph{	Although \algrucbp\, (Alg.~\ref{alg:rucbp})  yields the optimal regret guarantee, it requires an additional $O(K^2)$ space complexity to maintain the itemwise dominance sets $\cD_i(t)$ per item $i \in [K]$ at each round, along with the estimated pairwise preferences. Moreover it also requires an additional $O(|S_t|^2)$ computational complexity per round $t$ to update/maintain the dominance sets which could be as bad as $O(K^2)$. \algrmed\, (Alg.~\ref{alg:rmed}) only requires $O(|S_t|)$ computational complexity per round to compute $\hcB_i(t)$ and $\cI_i(t)$ for all $i \in S_t$.} 
% So clearly, \algrmed\, gives better runtime performance compared to \algrucbp, as also reflects from our experiemental evaluation (Sec.~\ref{sec:expts}, Fig.~\ref{fig:runtime}).}
\end{rem}
\fi %%%%%%%%%%%%%%%%%%%%%%%%%

\begin{rem}[\algrmed\, in standard dueling bandits]
\emph{Even in the dueling bandit setting (without the sleeping component), \algrmed~ and Thm.~\ref{thm:ub_rmed} have advantages compared to the RMED algorithm and analysis of \cite{Komiyama+15}. Our regret bound is valid for all number of items $K$, while the one of Thm.~3 of \cite{Komiyama+15} is only asymptotic when $K \to \infty$. 
This is due to the fact that the algorithm of \cite{Komiyama+15} depends on a hyper-parameter $f(K)$ which needs to larger than $A K$, where $A$ is a constant in $K$ and $T$ but which depends on the unknown sub-optimality gaps $\Delta(i,j)$. Thus, \cite{Komiyama+15} chooses $f(K) \approx K^{1+\epsilon}$ so that eventually the bound is satisfied when $K \to \infty$. Instead, our algorithm only depends on easily tunable hyper-parameters $t_0$ and $\alpha$, whose optimal values are independent of unknown parameters. }
\end{rem}

%%%%%%%%%%%%%%%%%%%%%%%%%%

%!TEX root = main.tex

\section{Experiments}
\label{sec:expts}

In this section, we compare the empirical performances of our two proposed algorithms (Alg.~\ref{alg:rucbp} and~\ref{alg:rmed}). Note that there are no other existing algorithms for our problem (see Sec.~\ref{sec:intro}). 
% The experiment details are given below: 

\textbf{Constructing Preference Matrices ($\P$).} 
We use the following three different utility based  Plackett-Luce$(\btheta)$ preference models (see Sec.~\ref{sec:prob}) that ensures a \emph{total-ordering}. We now construct three types of problem instances $1.$ {\it Easy} $2.$ {\it Medium} $3.$ {\it Hard}, for any given $K$, such that items with their respective $\btheta$ parameters are assigned as follows:
$1.$ {\it Easy:} $\btheta(1:\lfloor K /2 \rfloor) = 1$, $\btheta(\lfloor K/2 \rfloor + 1:K) = 0.5$.
$2.$ {\it Medium:} $\btheta(1:\lfloor K/3 \rfloor) = 1$, $\btheta(\lfloor K/3 \rfloor+1:\lfloor 2K/3 \rfloor) = 0.7$, $\btheta(\lfloor 2K/3 \rfloor + 1:K) = 0.4$.
$3.$ {\it Hard:} $\btheta(i) = 1 - (i-1)/K,\, \forall i \in [K]$. Note for each $\bsigma^* = (1 > 2> \ldots K)$. 
\vspace{-3pt}

%We now explain the different set of experiments and the inferences. 
In every experiment, we set the learning parameters $\alpha = 0.51$, $\delta = 1/T$ for \algrucbp \,(Alg.~\ref{alg:rucbp}) and as per Thm.~\ref{thm:ub_rmed} for \algrmed\, (Alg.~\ref{alg:rmed}). All results are averaged over $50$ runs.
% \vspace*{-10pt}

% \smallskip
\textbf{Regret over Varying Preference
Matrices.}
%\label{subsec:r_vs_t}
%\vspace{-5pt}
We first plot the cumulative regret of our two algorithms (Alg.~\ref{alg:rucbp} and~\ref{alg:rmed}) over time on the above three Plackett-Luce datasets for K = $10$. 
We generate availability sequence $\cS_T$ randomly by sampling every item $i \in [K]$ independently with probability $0.5$. 
Fig~\ref{fig:vs_t} shows that, as their names suggest too, \textit{instance-Easy} is easiest to learn as the best-vs-worst item preferences are well separated and the diversity of the item preferences across different groups are least. Consequently the algorithms yield slightly more regret on \emph{instance-Medium} due to higher preference diversity, and the hardest instance being \emph{Hard} where the learner really needs to differentiate the ranking of every item for any arbitrary set sequences $\cS_T$. 
Empirically \algrucbp\, is seen to slightly outperform \algrmed, though orderwise they perform competitively.

%\subsection
\textbf{Regret over Varying Set Availabilities.}
%\label{subsec:r_vs_s}
\vspace*{-1pt}
In these set of experiments, 
%we measure the regret over varying set sequences where some set of item subsets could be more frequently available over the rest. 
the idea is to understand how the regret improves over completely random subset availabilities as now the learner may not have to distinguish all item preferences as some of the item combinations occurs rarely.
We choose $K = 10$ and to enforce item dependencies we generate each set $S_t$ by drawing a random sample from Gaussian$(\bmu, \Sigma)$ such that $\bmu_i = 0,\, \forall i \in [10]$, and $\Sigma$ is a fixed $10 \times 10$ positive definite matrix which controls the set dependencies: Precisely we use two different block diagonal matrices for \emph{Low-Correlation} and \emph{High-Correlation} with the following correlations:
$1.$ \emph{Low-Correlation: } $\Sigma$ is a separated block diagonal matrix on item partitions $\{1,2,3\}$, $\{4,5,6\}$, $\{7,8,9,10\}$.
$2.$ \emph{High-Correlation: } $\Sigma$ is constructed by merging three all-$1$ matrices on partitions $\{1 \ldots 5\}$, $\{2, \ldots 8\}$, and $\{6, \ldots 10\}$, however as the resulting matrix is positive semi-definite, so we further take its SVD and reconstruct the matrix back eliminating the negative eigenvalues.
At every round we sample a random vector from Gaussian$(\bmu, \Sigma)$, and $S_t$ is considered to be the set of items whose value exceeds $0.5$.
Both experiments are run on \emph{instance-Hard}.
Fig.~\ref{fig:vs_S} shows, as expected, on \emph{Low-Correlation} both algorithms converge to $\bsigma^*$ relatively faster and at lower regret compared to \emph{High-Correlation} (as the later induces higher variability of the available subsets).

\vspace*{-10pt}
\begin{figure}[H]
	\begin{center} %[l b r t]
	\includegraphics[trim={3cm 0cm 4.5cm 0cm},clip,scale=0.2,width=0.25\textwidth]{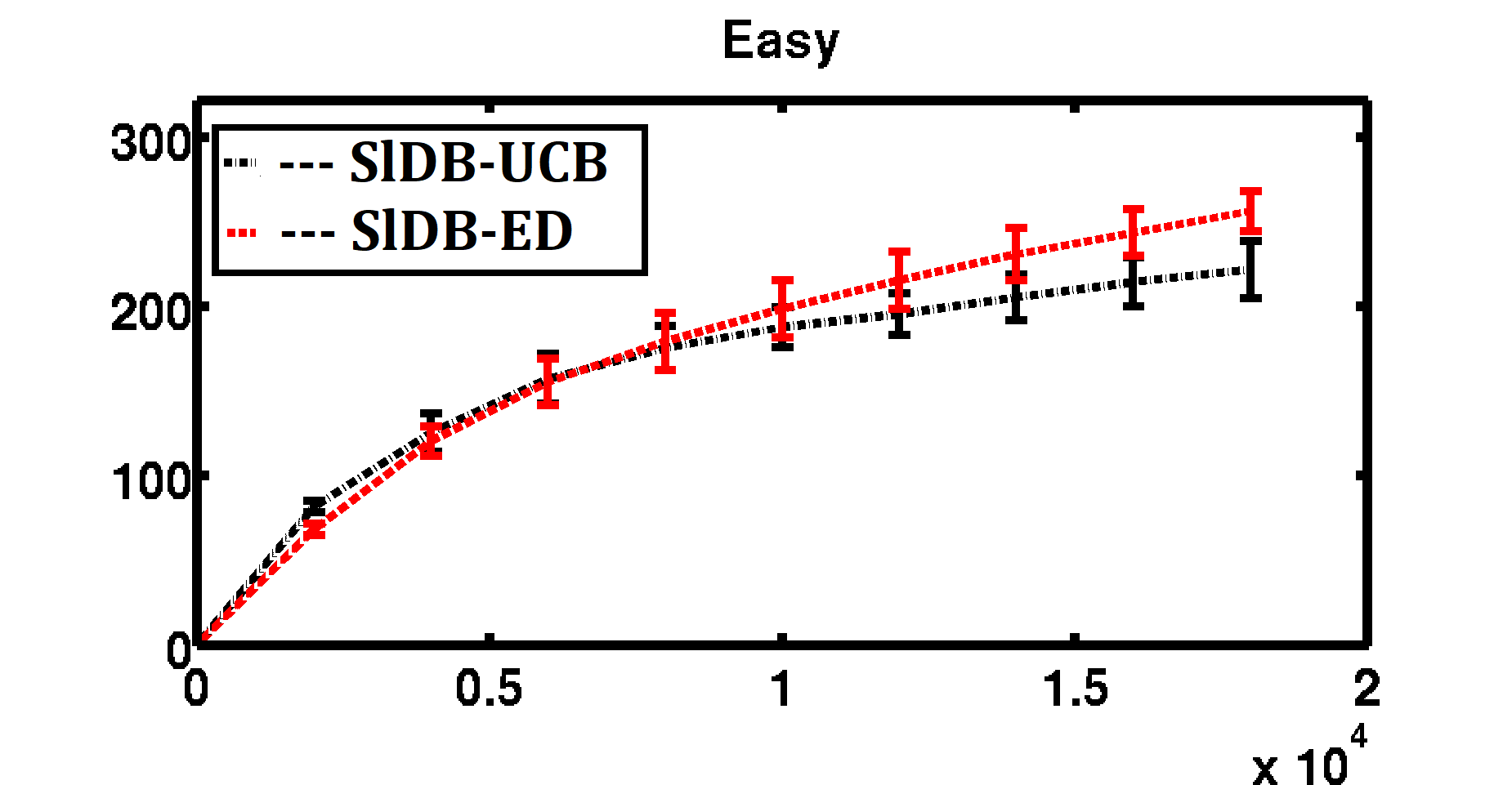}
		\includegraphics[trim={3cm 0cm 4.5cm 0cm},clip,scale=0.2,width=0.25\textwidth]{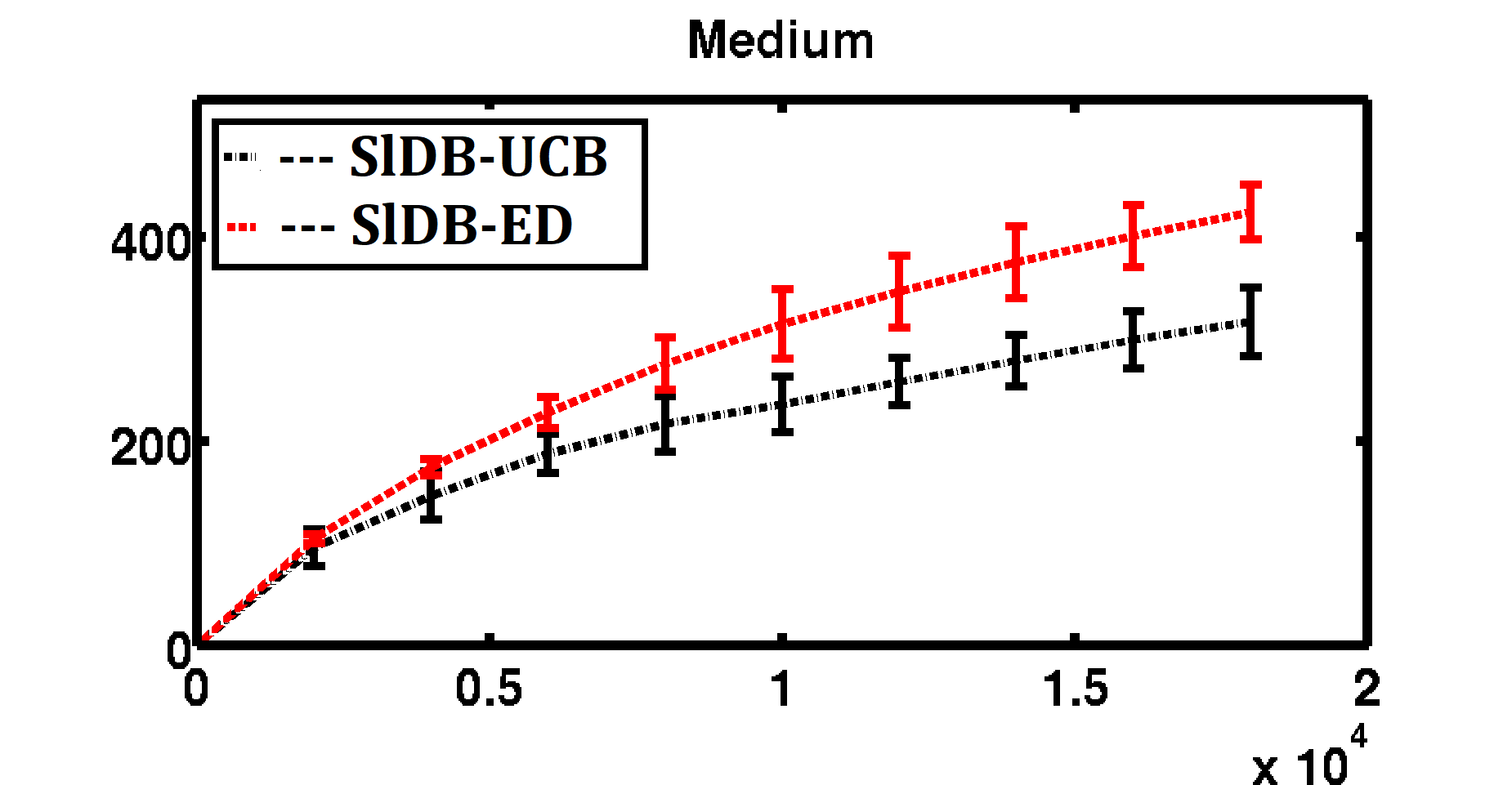}
		\includegraphics[trim={3cm 0cm 4.5cm 0cm},clip,scale=0.2,width=0.25\textwidth]{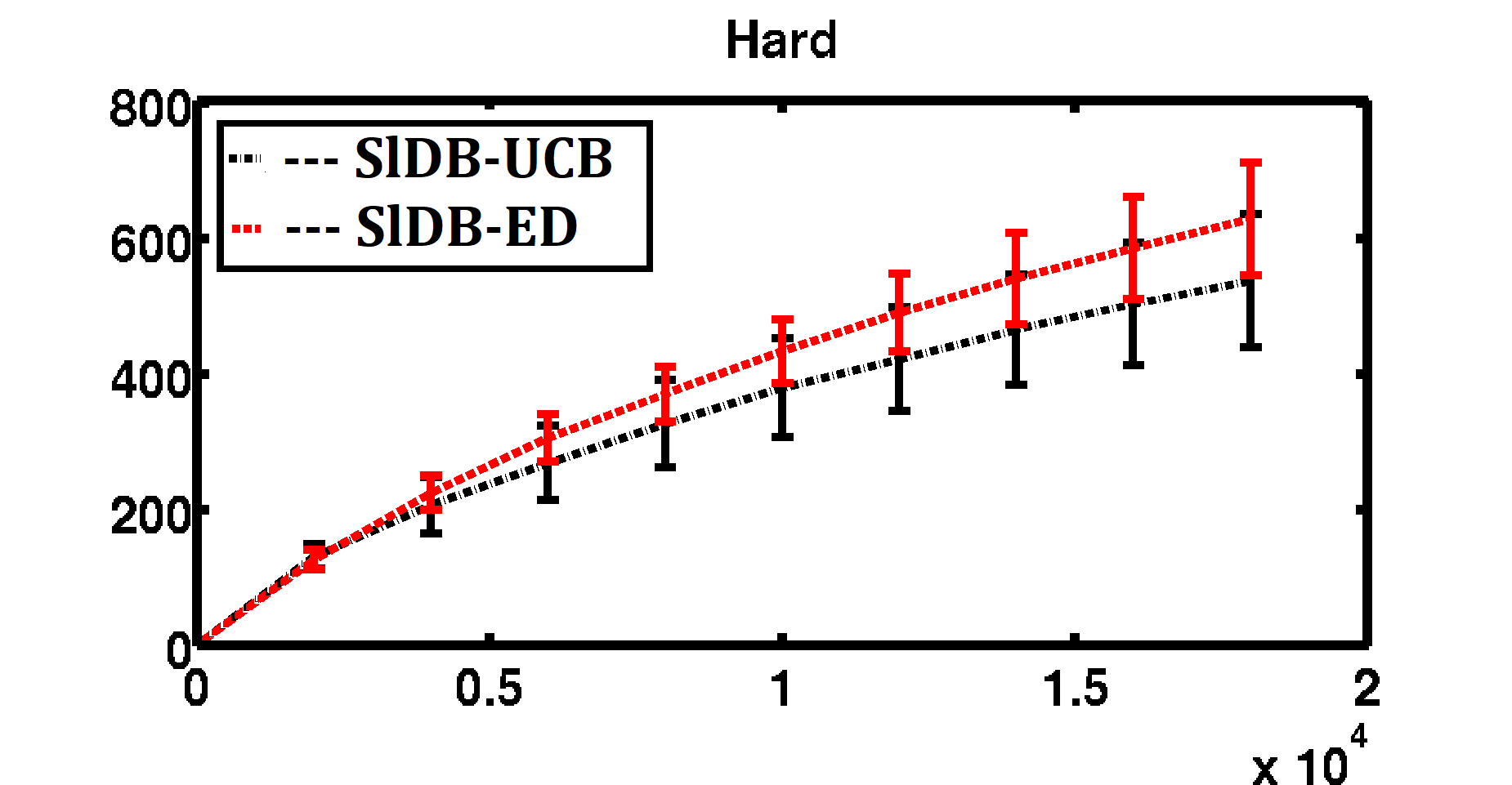}
		\vspace{-8pt}
		\caption{Regret ($R_t$) vs time ($t$) over three preference instances ($\P$)} \vspace*{-5pt}
		\label{fig:vs_t}
	\end{center}
\end{figure}
\vspace{-20pt}

\begin{minipage}{0.48\textwidth}
\vspace{-10pt}
 \begin{figure}[H]
	\begin{center}
%\hspace{-10pt}	
	\includegraphics[trim={3cm 0cm 4.5cm 0cm},clip,scale=0.1,width=0.47\textwidth]{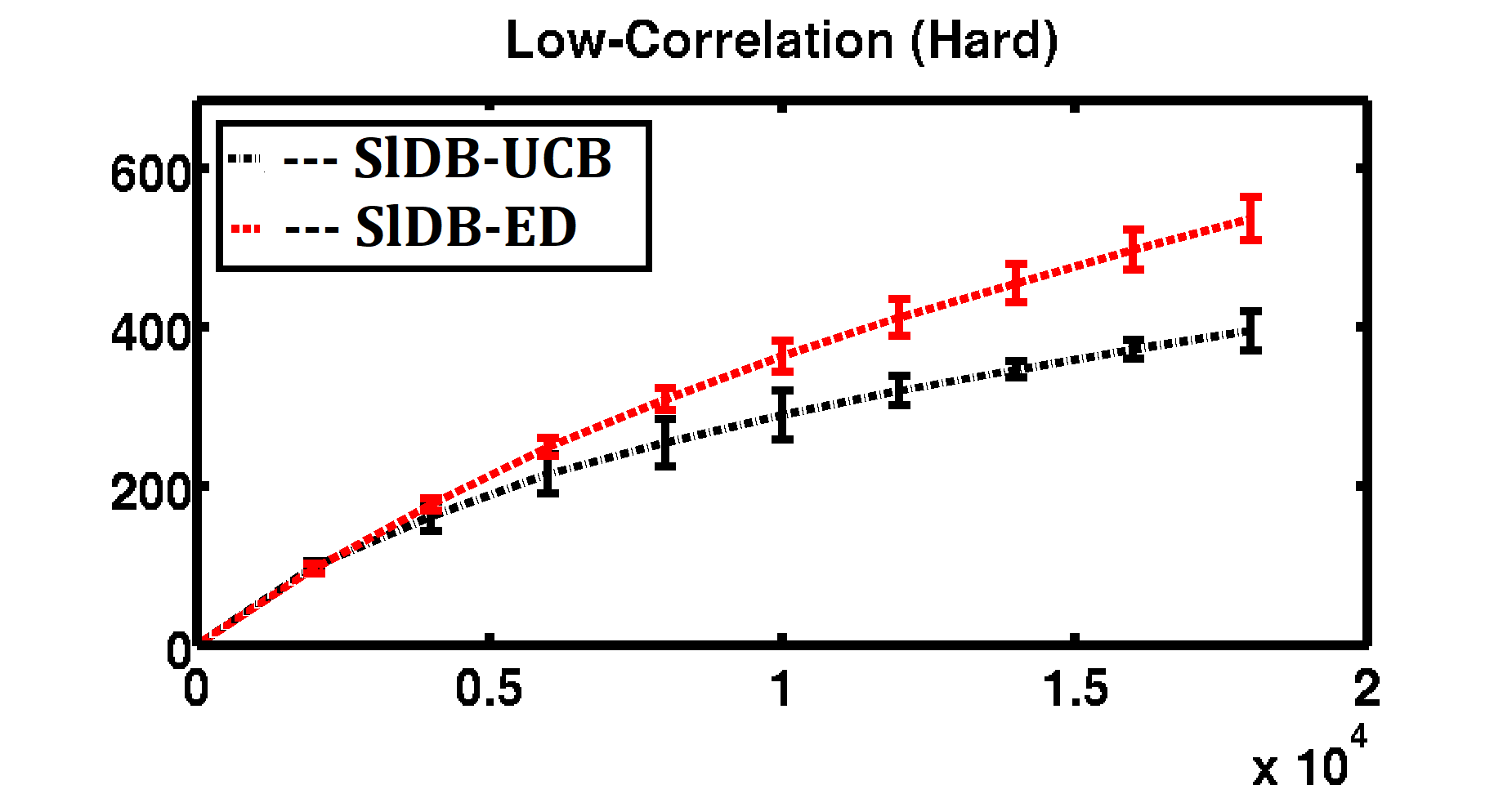}
	\vspace{3pt} 
		\includegraphics[trim={3cm 0cm 4.5cm 0cm},clip,scale=0.1,width=0.47\textwidth]{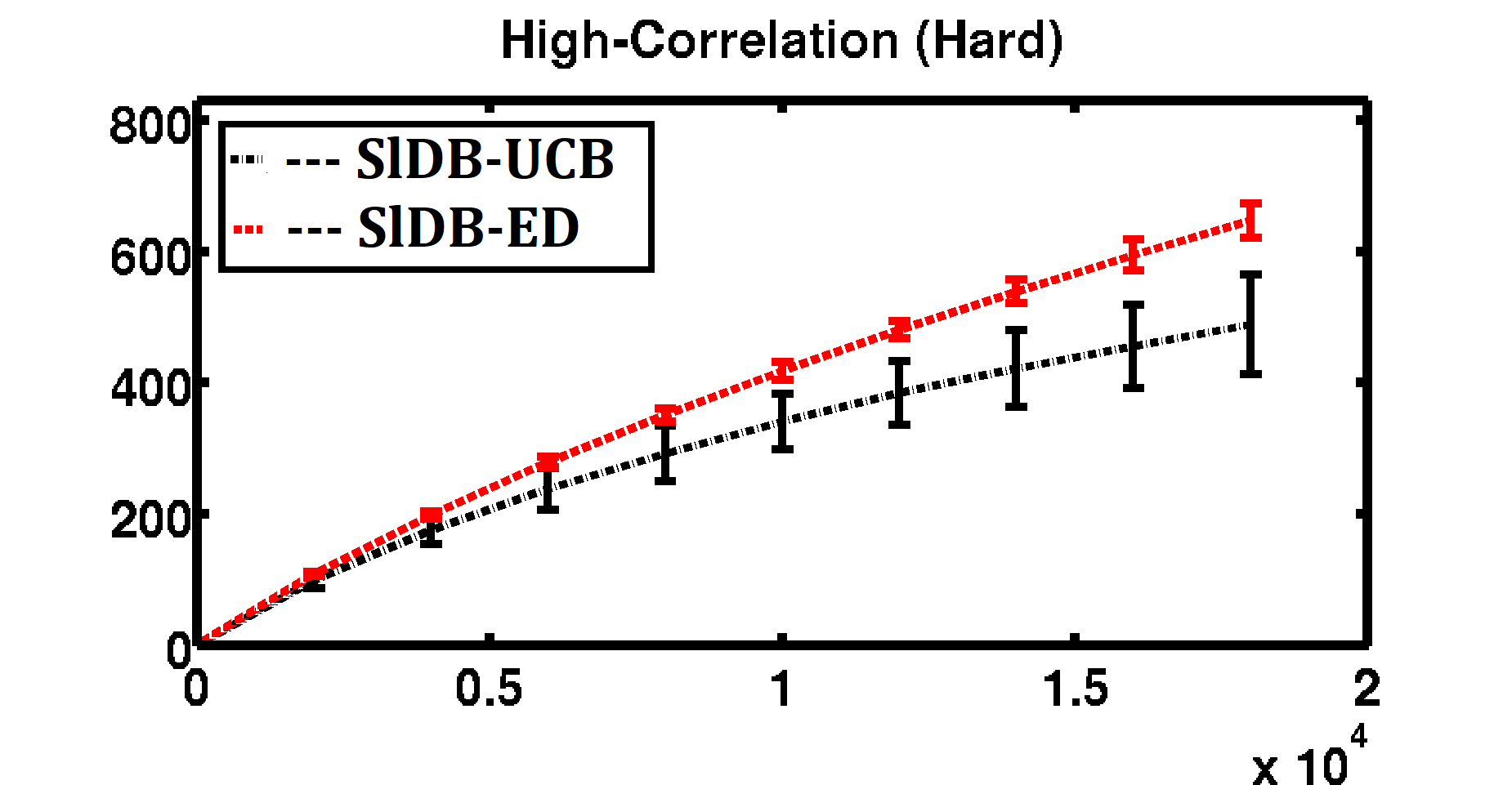} 
		\vspace*{-15pt}
		\caption{Regret ($R_t$) vs time ($t$) over availability sequences $\cS_T$}
		\label{fig:vs_S}
	\end{center}
\end{figure}
\vspace{-10pt}
\end{minipage}\quad
\begin{minipage}{0.48\textwidth}
\vspace{-10pt}
 \begin{figure}[H]
	\begin{center}
%\hspace{-10pt}	
	\includegraphics[trim={3.9cm 0cm 4.5cm 0cm},clip,scale=0.1,width=0.47\textwidth]{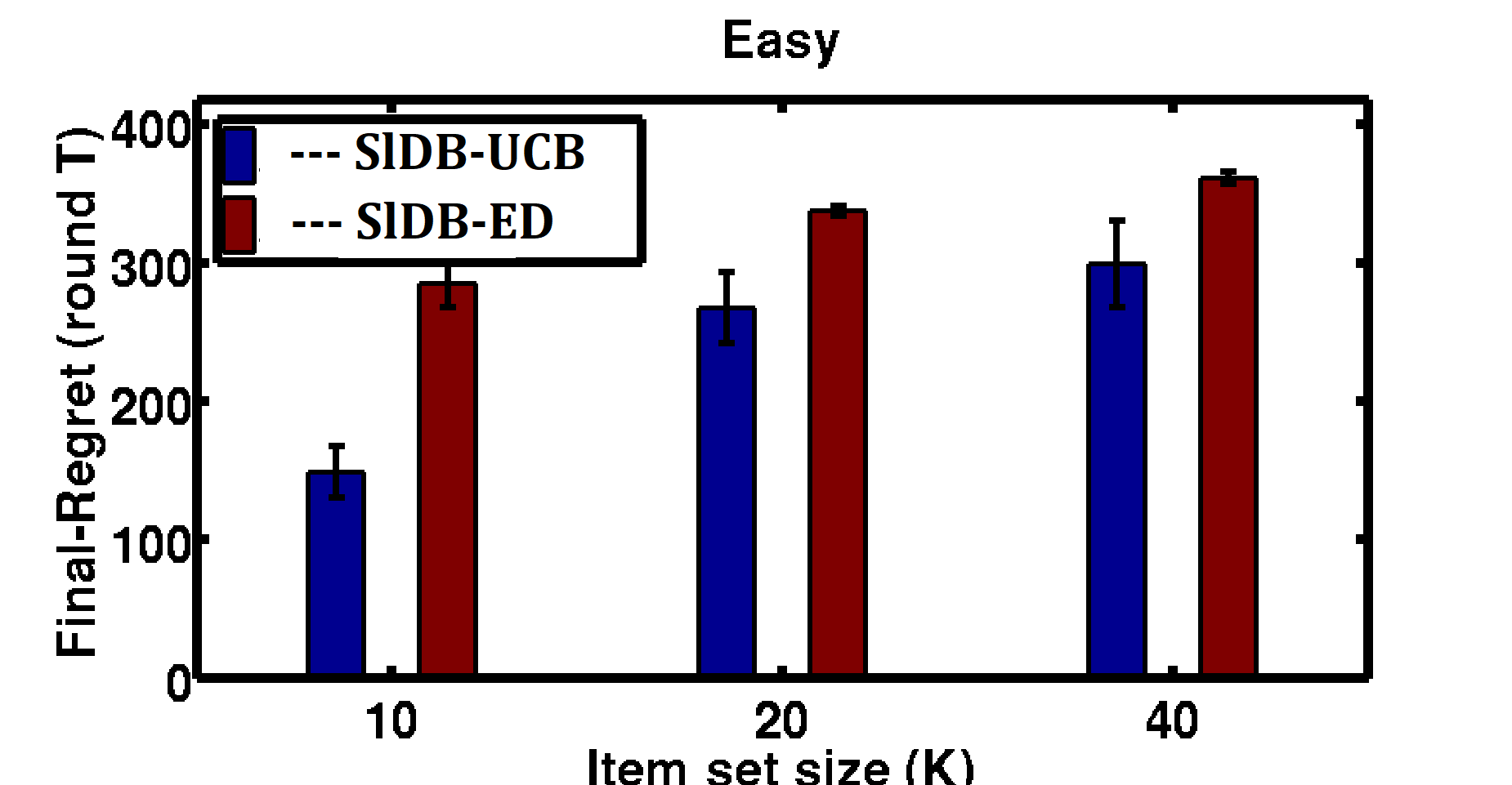}
		\includegraphics[trim={2.9cm 0cm 4.5cm 0cm},clip,scale=0.1,width=0.47\textwidth]{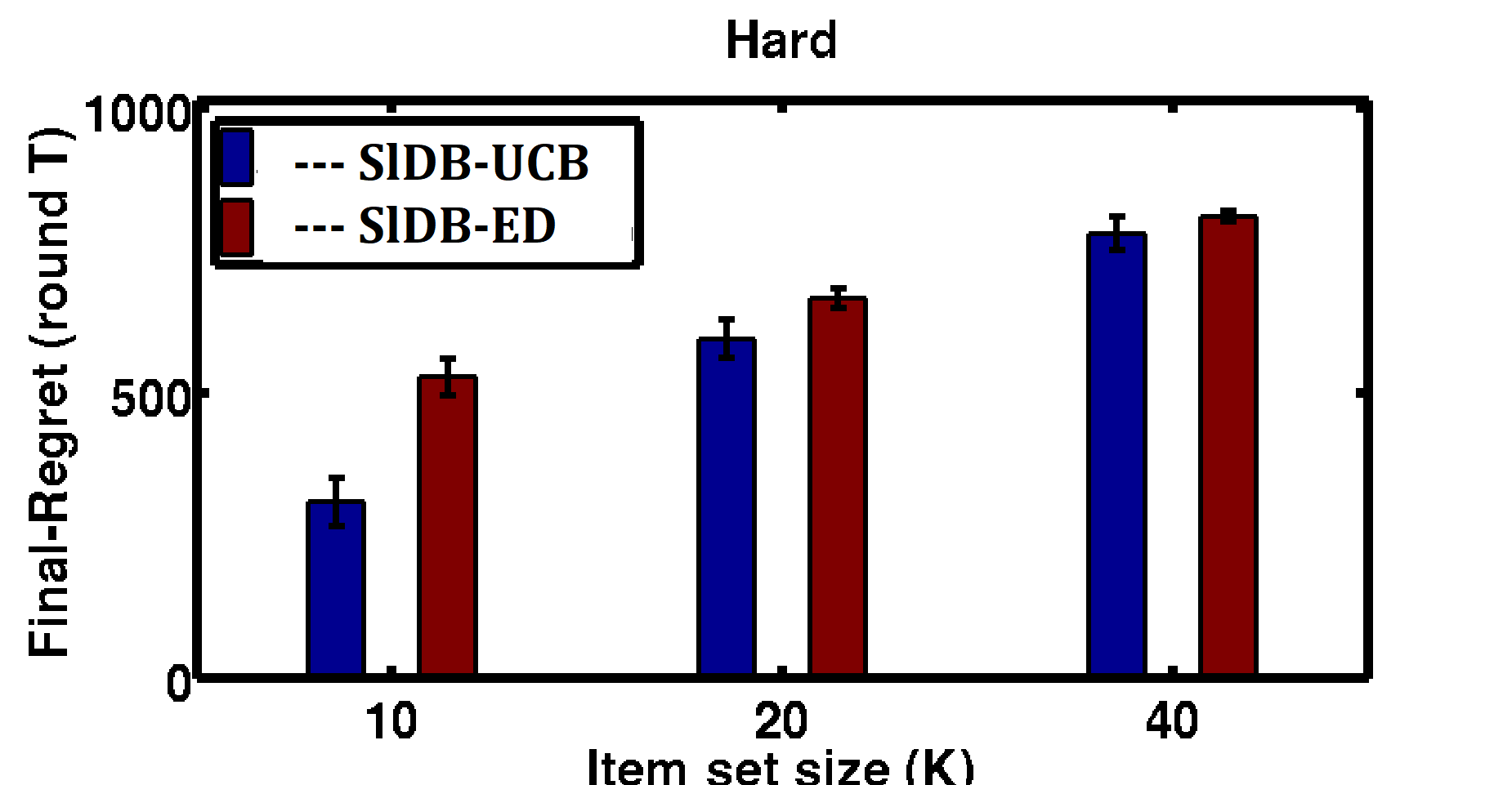} 
		\vspace*{-10pt}
		\caption{Final regret ($R_T$) at $T=10^4$ with varying sizes $(K)$ }
		\label{fig:vs_K}
	\end{center}
\end{figure}		
\vspace*{-10pt}
\end{minipage}

\vspace{-5pt}
%\subsection
\textbf{Final Regret vs Setsize(K).}
%\label{subsec: r_vs_k}
%\vspace{-5pt} 
We also compared the (averaged) final regret of the two algorithms over varying item sizes $K$. We additionally constructed two larger Plackett-Luce $(\btheta)$ \emph{Easy} and \emph{Hard} instances for  $K=20$ and $40$, using similar $\btheta$ assignments explained before. We set $T = 10\,000$ and use itemwise independent set generation idea, as described for Fig.~\ref{fig:vs_t}. As expected, Fig.~\ref{fig:vs_K} shows the regret of both algorithms scales up with increasing $K$ with effect on \algrmed\, being slightly worse than \algrucbp, though the later generally exhibits a higher variance.% in almost all cases.

%\subsection
\textbf{Worst Case Regret vs Time.}
We run an experiment to analyse the regret of our two algorithms on the worst case problem instances. Towards this we use preference matrices $\P_{\Delta}$ of the form: $\P_{\Delta}(i,j) = 0.5 + \Delta, ~\forall 1 \le i < j \le K$, i.e. all items are spaced with equidistant gap $\Delta \in (0,0.5]$. As before, we choose $T = 20,000$ and $K = 10$, and run the algorithms on above problem instances varying $\Delta$ in the range $[10/T, \ldots, 0.5]$ with uniform grid-size of $0.005$ (i.e. total $100$ values of $\Delta$, each corresponds to a separate problem instance $\P_{\Delta}$ with different `gap-complexity'). At the end we plot the worst case regret of both the algorithms over time, by plotting $\max_{\Delta}R_t(P_{\Delta})$ vs $t$. We run the experiments over three availability sequences: 1. Independent (as used in Fig.~\ref{fig:vs_t}), 2. Low-Correlation, and 3. High-Correlation (as used in Fig.~\ref{fig:vs_S}). As a consequence the resulting plots reflect the worst case (w.r.t. $\Delta$) performances of the algorithms, which seem to be scaling as $O(T^{2/3})$ for \algrmed, as conjectured to be its distribution free upper bound (see discussion after Thm.~\ref{thm:ub_rmed}), and with a slightly lower rate for \algrucbp. Fig.~\ref{fig:vs_wc} shows the comparative performances.
\vspace{-5pt}
\begin{figure*}[h]
	\begin{center}\includegraphics[trim={3cm 0cm 4.5cm 0cm},clip,scale=0.245,width=0.325\textwidth]{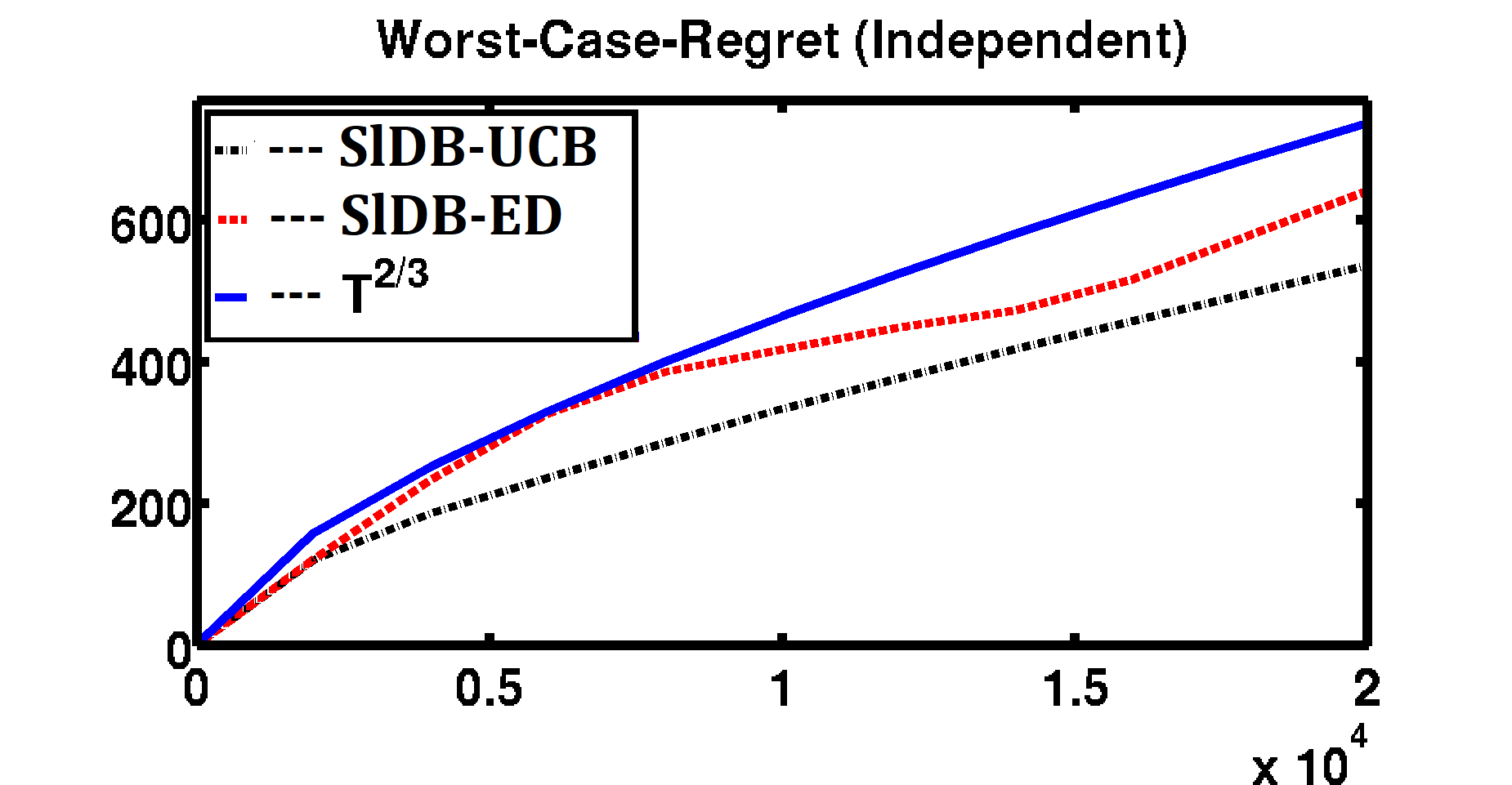}
	\vspace{3pt} 
		\includegraphics[trim={3cm 0cm 4.5cm 0cm},clip,scale=0.245,width=0.325\textwidth]{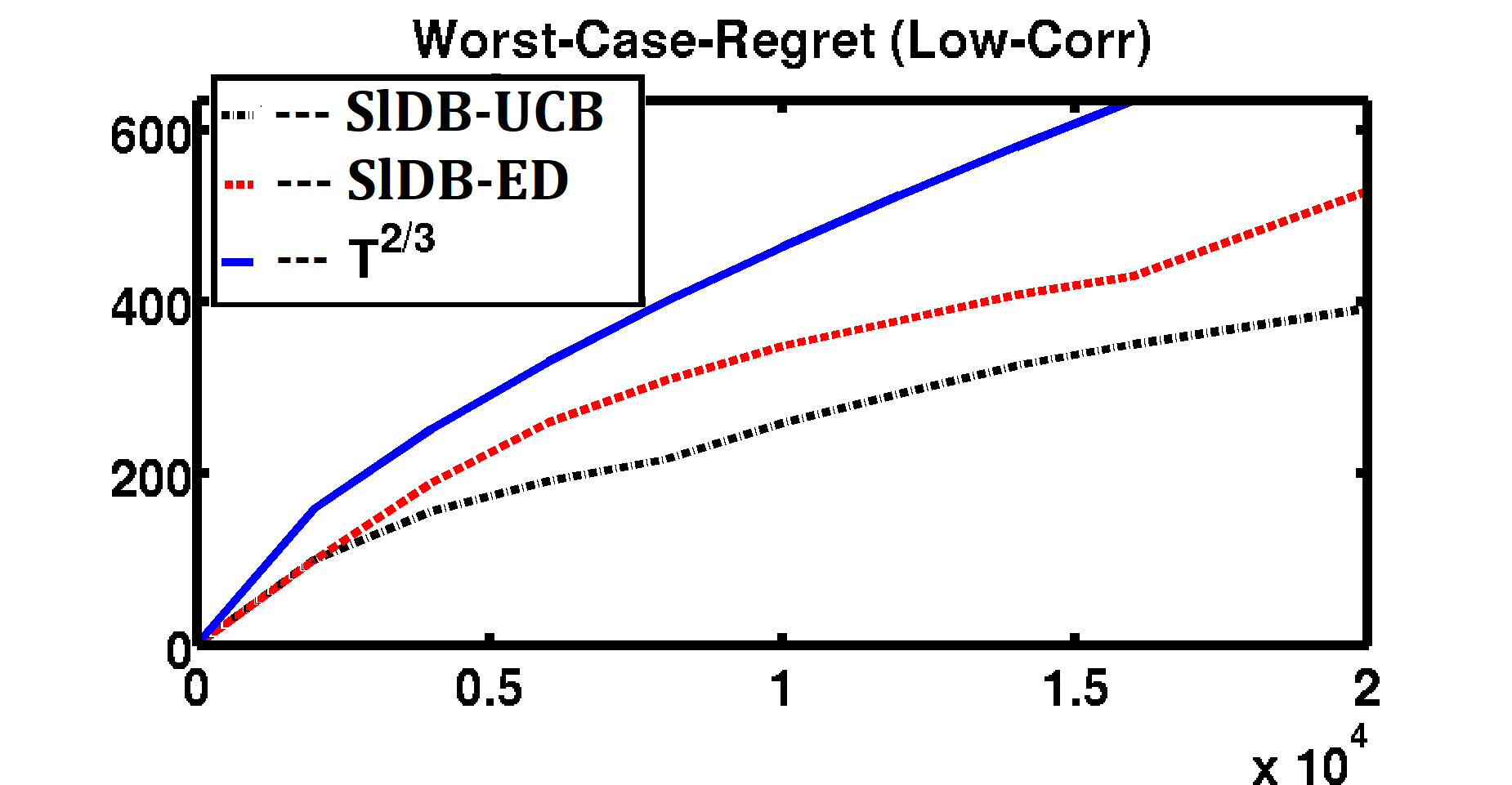} 
		\includegraphics[trim={3cm 0cm 4.5cm 0cm},clip,scale=0.245,width=0.325\textwidth]{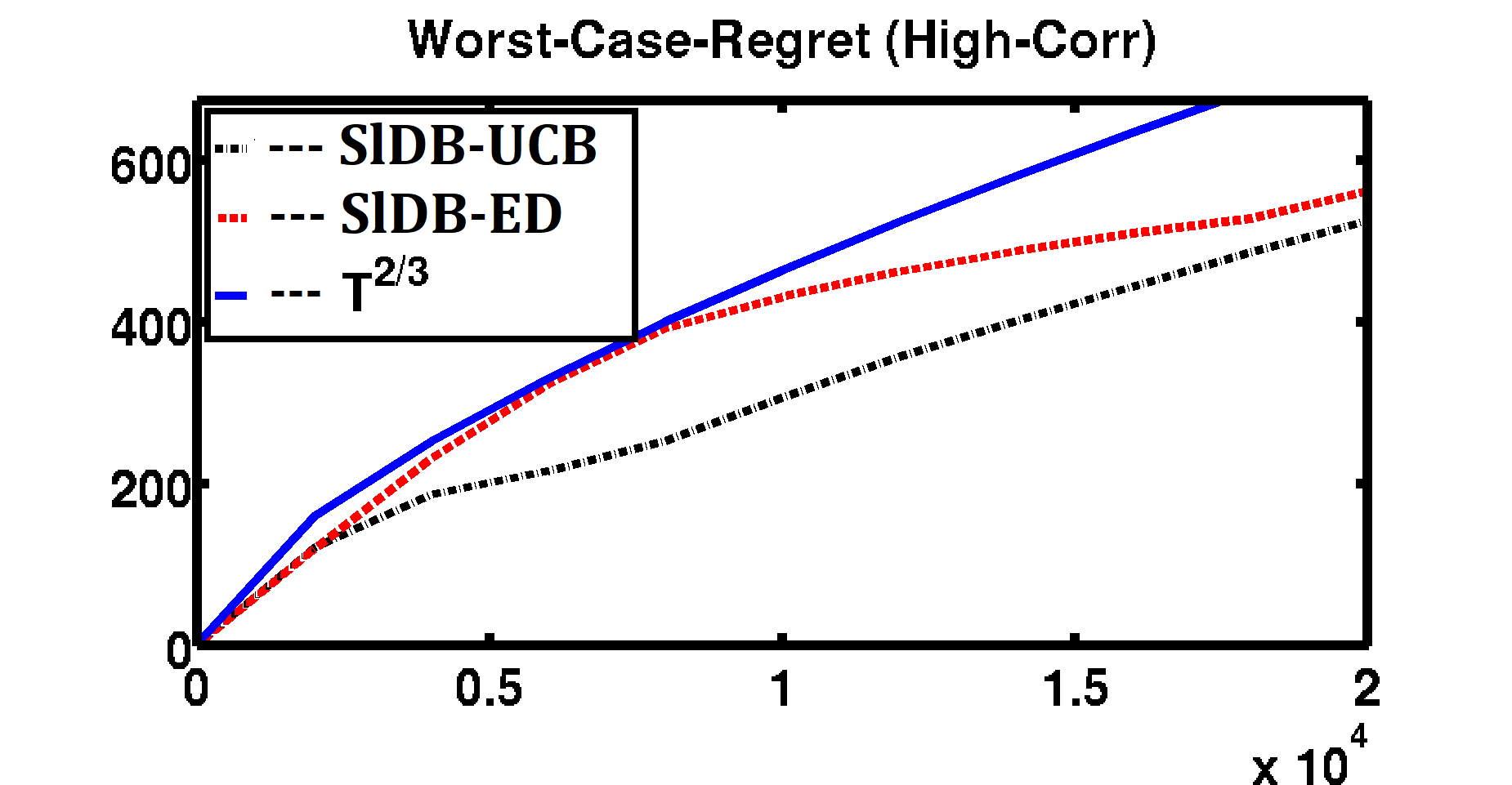}
		\vspace*{-10pt}
		\caption{``Worst Case Regret" ($\max_{\Delta}R_t(P_{\Delta})$) vs time ($t$) over three availability sequences $\cS_T$}
		\label{fig:vs_wc}
	\end{center}
\end{figure*}
\vspace{-5pt}

%%%%%%%%%%%%%%%%%%%%%%%%%%%%%%%%%%%%%
%We present more detailed empirical evaluations and discussions in Appendix~\ref{app:expts}. 

%!TEX root = main.tex

\section{Conclusion and Perspective}
\label{sec:concl}
We introduce the problem of sleeping dueling bandits with stochastic preferences and adversarial availabilities, which, despite of great practical relevance, was left unaddressed till date. 
Towards this we adapt two dueling bandit algorithms for the problem and give regret analysis for both. We also derive an instance dependent regret lower bound for our problem setup which shows that our second algorithm is asymptotically near-optimal (up to the problem dependent constants). Finally, we compare both our algorithms empirically where usually the first algorithm is shown to outperform the second, although having a relatively weaker regret.

\textbf{Future Works.} Moving forward, one can address many open questions along this direction, including relaxing the \emph{total-ordering} assumption on the stochastic preferences assuming more general ranking objective based on {borda} \cite{SAVAGE} or {copeland} scores \cite{Zoghi+15}, or extending the framework to a general contextual scenario with subsetwise feedback. Another direction worth understanding is to analyze the connection of this problem with other bandit setups, e.g., learning with feedback graphs \cite{Alon+15,Alon+17} or other side information \cite{SideInfo11,SideInfo14}. It would also be interesting to consider the dueling bandit problem for adversarial preference and stochastic availabilities \cite{neu14,kanade09}, and also analyzing these class of problems for general subsetwise preferences \cite{Ren+18,sui2018advancements,Brost+16}.
%Detailed discussion is given in Appendix~\ref{app:concl}.

%We introduce the problem of sleeping dueling bandits with stochastic preferences and adversarial availabilities, which, despite of great practical relevance, was left unaddressed till date. 
%
%Towards this we adapt two dueling bandit algorithms for the problem and give regret analysis for both. We also derive an instance dependent regret lower bound for our problem setup which shows that our second algorithm is asymptotically near-optimal (up to the problem dependent constants). Finally, we compare both our algorithms empirically where usually the first algorithm is shown to outperform the second, although having a relatively weaker regret.
%
%Moving forward, one can address many open questions along this direction, including relaxing the \emph{total-ordering} assumption on the stochastic preferences assuming more general ranking objective based on {borda} \cite{SAVAGE} or {copeland} scores \cite{Zoghi+15}, or extending the framework to a general contextual scenario with subsetwise feedback. Another direction worth understanding is to analyze the connection of this problem with other bandit setups, e.g., learning with feedback graphs \cite{Alon+15,Alon+17} or other side information \cite{SideInfo11,SideInfo14}. It would also be interesting to consider the dueling bandit problem for adversarial preference and stochastic availabilities \cite{neu14,kanade09}, and also analyzing these class of problems for general subsetwise preferences \cite{Ren+18,sui2018advancements,Brost+16}.

\newpage

\bibliographystyle{plain}
\bibliography{bib_sleepingDB}

\newpage
%!TEX root = icml-sleeping.tex
\onecolumn{
\section*{\centering\LARGE{Supplementary: \papertitle}}
\vspace*{1cm}
	
\allowdisplaybreaks

\section{Appendix for Sec.~\ref{sec:lb}}
\label{app:lb}

\begin{defn}[\nr\, algorithm]
\label{def:con}
An algorithm $\cA$ for \dspaa\, problem is defined to be a \nr\, algorithm, if for each problem instance \SPAA$(\P,\cS_T)$ model, the expected number of times $\cA$ plays any suboptimal duel $(i,j) \in [K]\times[K]$ is sublinear in $T$, or more precisely, $\forall (i,j) \neq (i_t^*,i_t^*)$, $\E[n_{ij}(T)] = o(T^\alpha)$, for some $\alpha \in (0,1)$, where recall that we define $n_{ij}(t):= \sum_{\tau = 1}^{t}\1\big(\{x_t,y_t\} = \{i,j\}\big)$ denotes the number of times the pair $(i,j)$ is played by $\cA$ in $T$ rounds. ($\E[\cdot]$ denotes expectation under the randomization of $\cA$ and the \SPAA$(\P,\cS_T)$ model.)
\end{defn}

\subsection{Proof of Thm.~\ref{thm:lb}}

\begin{proof}
The main argument lies behind the fact that in the worst case the adversary can force the algorithm to learn the preference of every distinct pair $(i,j)$ as the in the `worst-case' sequence $\cS_T$ knowledge of the already `learnt' pairwise preferences would not disclose any information on the remaining pairs; e.g. assuming $\bsigma^* = (1,2,\ldots, K)$, revealing the available subsets in the following sequence $(1,2), (1,3), \ldots (1,K), (2,3), (2,4), \ldots (K-1,K)$ would force the learner to explore (learn the preferences) all $K \choose 2$ distinct pairs. 

The remaining proof establishes this formally, towards which we first show a $\Omega(\frac{\ln T}{\Delta(1,2)})$ regret lower bound for a \SPAA \, instance with just two items (i.e. $K=2$) as shown in Lem.~\ref{lem:lb}. The lower bound for any general $K$ can now be derived applying the above bound on independent $K \choose 2$ subintervals, with the availability sequence $(1,2), (1,3), \ldots (1,K), (2,3), (2,4), \ldots (K-1,K)$. 

For the interest of the problem instance construction to prove the lower bound, we would assume $\Delta(i,i+1)>0, ~\forall i \in [K-1]$ and thus we use $\Delta(i,i+1)_+ = \Delta(i,i+1)$ for the rest of this proof (as also assumed for Lem.~\ref{lem:lb}). Note that this is without loss of generality since otherwise the regret lower bound in Lem.~\ref{lem:lb} is trivially $0$.

%The complete proof is given in Appendix~\ref{app:lb}.
We add the details below for completeness.

Let $K' = $ $K \choose 2$ and suppose we divide the time horizon into sub-intervals $1,2,\ldots K'$ each of length $T':= T/K'$, where the available subsets are fixed inside every subinterval, and follows the sequence $(1,2), (1,3), \ldots (1,K), (2,3), (2,4), \ldots (K-1,K)$ across subintervals. 
Note that with above construction, the regret minimization problem within each sub-interval boils down to the standard stochastic dueling bandit problem over $2$ arms. 
%and thus using the lower bound result of Lem.~\ref{lem:lb} we get that the expected regret of $\cA$ within any sub-interval $\ell$ is $E[R_{\ell}(\cA)]:= \Omega\bigg(\frac{\log T_{sub}}{\Delta(2\ell-1, 2\ell)}\bigg)$. 
Further since the preferences of available set $S_t$s are independent across different sub-intervals, applying the lower bound of Lem.~\ref{lem:lb} individually to every $K'$ subintervals the total cumulative regret of $\cA$ in $T$ rounds can be lower bounded as:
\begin{align*}
\E[R_T(\cA)] = \sum_{i = 1}^{K-1}\sum_{j = i+1}^K\E[R_{T'}(\cA)] \ge  \Omega\bigg(\sum_{i = 1}^{K-1}\sum_{j = i+1}^K \frac{\log T'}{\Delta(i, j)}\bigg) = \Omega\bigg(\sum_{i = 1}^{K-1}\sum_{j = i+1}^K \frac{\log T}{\Delta(i, j)}\bigg)
\end{align*}

where the last inequality holds since $T \ge (K')^2$, which implies $\log \frac{T}{K'} \ge \log T - \log \sqrt T = \nicefrac{1}{2}\log T$, and this concludes the proof.

\iffalse % % % % % % % % % % % % % % % % % % %

Similarly, instead if we would  have divide the time horizon into $\tau' := \frac{K}{2}-1$ sub-intervals each of length $T'_{sub}:= \frac{T}{\tau}$, where at any round $t$ inside a sub-interval $\ell' \in [\tau']$, the set of available items are $S_t = (2\ell', 2\ell'+1)$. Then again using the lower bound result of Lem.~\ref{lem:lb} we get that the expected regret of $\cA$ within any sub-interval $\ell'$ is $E[R_{\ell'}(\cA)]:= \Omega\bigg(\frac{\log T'_{sub}}{\Delta(2\ell, 2\ell+1)}\bigg)$, and similarly since the available set $S_t$s are independent across different sub-intervals, the total cumulative regret of $\cA$ in $T$ rounds can be lower bounded as:
\[
\E[R_T(\cA)] = \sum_{\ell'=1}^{\tau'}\E[R_{\ell'}(\cA)] \ge  \Omega\bigg(\sum_{\ell'=1}^{\tau'}\frac{\log T'_{sub}}{\Delta(2\ell', 2\ell'+1)}\bigg) = \Omega\bigg(\sum_{\ell'=1}^{\tau'}\frac{ \log \frac{2T}{K+2}}{\Delta(2\ell', 2\ell'+1)}\bigg) = \Omega\bigg(\sum_{\ell'=1}^{\frac{K}{2}-1}\frac{\log \frac{T}{K}}{\Delta(2\ell', 2\ell'+1)}\bigg),
\]
for any $K \ge 2$.
Then taking the average over the above two derived lower bounds:
\[
\E[R_T(\cA)] = \nicefrac{1}{2}\Bigg[\Omega\bigg(\sum_{\ell=1}^{\frac{K}{2}}\frac{\log \frac{T}{K}}{\Delta(2\ell-1, 2\ell)}\bigg)
 + \Omega\bigg(\sum_{\ell'=1}^{\frac{K}{2}-1}\frac{\log \frac{T}{K}}{\Delta(2\ell', 2\ell'+1)}\bigg)\Bigg] = \Omega\bigg(\sum_{\ell=2}^{K}\frac{\log T}{\Delta(\ell-1, \ell)}\bigg),
\]

\fi % % % % % % % % % % % % % % % % % % % % %
\end{proof}

\lemlb*

\begin{proof}
Note that for $K=2$, the only non-trivial available set is $\{1,2\}$, therefore we assume $S_t = \{1,2\}, ~\forall t \in [T]$.
The proof now simply follows by applying the existing lower bound (Thm. $2$) of \cite{Yue+12} for standard stochastic dueling bandit problem for only $2$ arms.
\end{proof}

% % % % % % % % % % % % RUCB % % % % % % % % % % % %

%!TEX root = main.tex

\section{Appendix for Sec.~\ref{sec:algo_rucbp}}
\label{app:rucbp}

\noindent \textbf{Notations.}
Let us start with defining useful notation for the analysis. We write for any pair $1\leq i <j \leq K$
\[
   M_{ij} = \sum_{k=1}^i \frac{4\alpha}{\min\big\{\Delta(k,i)_+ ,\Delta(k,j)_+ \big\}^2} \,.
\]
We also denote $S_{\sm i} = S \sm \{i\}, \, i \in S$, for any $S \subseteq [K]$.

\subsection{Complete proof of Thm.~\ref{thm:ub_rucb}}

%\ubrucbp*

\begin{customthm}{\ref{thm:ub_rucb}}
Given any $\delta > 0$ and $\alpha\ge1$, with probability at least $1-\delta$, the regret incurred by \algrucbp ~(Alg.~\ref{alg:rucbp})  
is upper-bounded as: 
\vspace*{-3pt}
\begin{align*}
R_T  & \leq 2 \sum_{i=1}^{K-1} \sum_{j=i+1}^K M_{ij} \log \big(2 C(K,\delta) M_{ij}\big)  
\end{align*}
where  $C(K,\delta) := \big({(4\alpha-1) K^2}/((2\alpha-1)\delta)\big)^{\frac{1}{2\alpha-1}}$.
\end{customthm}

% \begin{rem}
% 	We present a much tighter analysis of Thm.~\ref{thm:ub_rucb} with just $O(1)$ regret compared to the $O(K^2\log T/\Delta)$ upper bound presented in main draft. It is crucial to note that, it does not contradict our lower bound result (Thm.~\ref{thm:lb}), which is a lower bound on the `expected regret' on `worst-case availability sequence'. In contrast, since Alg.~\ref{alg:rucbp} is meant to work on the fixed confidence setting (i.e. given a fixed $\delta$ as input parameter), it can to an carefully designed efficient explore-then-commit strategy, which can trace-down the best-arm of any set $S$ in $O(1)$ pairwise samples, with probability at least $(1-\delta)$. Thus committing to best-arm of each set there-after subsequently cost no regret, resulting to an $O(1)$ regret algorithm overall. 
% \end{rem}

\begin{proof}[Proof of Thm.~\ref{thm:ub_rucb}]
%We start with describing an intuitive proof sketch. 
The key steps lie in proving the following four lemmas.
The first lemma follows along the line of Lem.~1 of RUCB algorithm~\cite{Zoghi+14RUCB}. It shows after $C(K,\delta)$ rounds all the pairwise estimates are contained within their respective confidence intervals: 

\begin{customlemma}{\ref{lem:conf_cdels}}
Let $\alpha >0.5$ and $\delta >0$. Then, with probability at least $1-\delta$, for any $i,j \in [K]$%
\[%\begin{equation}
	     \hp_{ij}(t)-c_{ij}(t) \leq p_{ij} \leq  u_{ij}(t) := \hp_{ij}(t)+c_{ij}(t),  \qquad \forall t \in [T] \,.
\]
\end{customlemma}

The lemma below is adapted from Proposition~2, ~\cite{Zoghi+14RUCB}. It basically states that once the algorithm has explored enough (i.e., more than $C(K,\delta)$) the algorithm will not play a suboptimal pair too many times. %The proof is deferred to Appendix~\ref{app:rucbs}. 

\begin{customlemma}{\ref{lem:rucb_nijs}}
	Let $\alpha > 0.5$. Under the notations and the high-probability event of Lem.~\ref{lem:conf_cdels}, for all $i,j,k \in [K]$ such that $\{i,j\} \neq \{k,k\}$, and for any $\tau \geq 1$
	\begin{equation*}
		\sum_{t = 1}^{\tau} \1(i_t^* = k)\1\big(\{x_t,y_t\} = \{i,j\}\big) 
		\le \frac{4\alpha \log a_{i,j}(\tau)}{\min\big\{\Delta(k,i)_+ ,\Delta(k,j)_+ \big\}^2} \,,
	\end{equation*}
where recall $a_{ij}(\tau) = \max\big( C(K,\delta),n_{ij}(\tau)\big)$.
\end{customlemma}

Given the above results, we are  ready to analyze the regret guarantee of \algrucbp. 
For ease on notation we denote $\cX_t = \{x_t,y_t\}$.
Let us assume the `good event' of Lem.~\ref{lem:conf_cdels} holds good for all $t\in[T]$, which is true with probability of at least $1-\delta$. Conditioned on that, note that Lem.~\ref{lem:rucb_nijs} is satisfied. Based on this we now analyze the regret of Alg.~\ref{alg:rucbp}:
\begin{align}
 R_T  & = \sum_{t=1}^T\sum_{k = 1}^{K-1}\1(i_t^* = k)r_t \nonumber \\
  & = \sum_{t=1}^T\sum_{k = 1}^{K}   \sum_{i=k}^K \sum_{j=i}^K \1(i_t^* = k)  \1\big(\{x_t,y_t\} = \{i,j\}\big)r_t  \qquad \leftarrow \ \text{because $x_t \geq k$ and $y_t \geq k$} \nonumber \\
  & = \sum_{t=1}^T\sum_{k = 1}^{K-1}   \sum_{i=k}^{K-1} \sum_{j=i+1}^K \1(i_t^* = k)  \1\big(\{x_t,y_t\} = \{i,j\}\big) r_t   \ \leftarrow \ \text{because $i=j=k$ implies $r_t = 0$} \nonumber \\
  & \leq \sum_{t=1}^T\sum_{k = 1}^{K-1}   \sum_{i=k}^{K-1} \sum_{j=i+1}^K \1(i_t^* = k)  \1\big(\{x_t,y_t\} = \{i,j\}\big) \ \leftarrow \ \text{because $r_t \leq 1$} \nonumber \\
  & = \sum_{i=1}^{K-1} \sum_{j=i+1}^K \sum_{t=1}^T \sum_{k=1}^i \1(i_t^* = k)  \1\big(\{x_t,y_t\} = \{i,j\}\big) \nonumber \\
	& = \sum_{i=1}^{K-1} \sum_{j=i+1}^K n_{ij}(T) \,.
	\label{eq:sumnij}
\end{align}

Now, fix $1\leq i < j \leq K$ and let us upper-bound $n_{ij}(T)$ the number of times such a pair is played. Summing the upper-bound of Lemma~\ref{lem:rucb_nijs} over $k \leq i$, we get
\[
	n_{ij}(T) =  \sum_{k=1}^i \sum_{t=1}^T \1(i_t^* = k)  \1\big(\{x_t,y_t\} = \{i,j\}\big)  \leq \sum_{k=1}^i \frac{4\alpha \log(\max\{C(K,\delta), n_{ij}(T)\})}{\min\big\{\Delta(k,i)_+ ,\Delta(k,j)_+ \big\}^2}  \,.
\]
Therefore, since $C(K,\delta) \geq 1$,
\[
	n_{ij}(T) \leq M_{ij} \big(\log (C(K,\delta) + \log (n_{ij}(T))\big), \qquad \text{where}\quad M_{ij} = \sum_{k=1}^i \frac{4\alpha}{\min\big\{\Delta(k,i)_+ ,\Delta(k,j)_+ \big\}^2}\,.
\]
which implies 
\[
	n_{ij}(T) \leq 2 M_{ij} \big(\log C(K,\delta) + \log (2 M_{ij})\big) \,.
\]
Substituting into Inequality~\eqref{eq:sumnij} entails
\[
	R_T  \leq 2 \sum_{i=1}^{K-1} \sum_{j=i+1}^K M_{ij} \log \big(2 C(K,\delta) M_{ij}\big) \,,
\]
which concludes the proof.
\end{proof}

\subsection{Technical lemmas for Thm.~\ref{thm:ub_rucb}}

\begin{customlemma}{\ref{lem:conf_cdels}}
Let $\alpha >0.5$ and $\delta >0$. Then, with probability at least $1-\delta$, for any $i,j \in [K]$%
\[%\begin{equation}
	     \hp_{ij}(t)-c_{ij}(t) \leq p_{ij} \leq  u_{ij}(t) := \hp_{ij}(t)+c_{ij}(t),  \qquad \forall t \in [T] \,.
\]
\end{customlemma}

\begin{proof}
	The proof of this lemma is adapted from a similar result (Lemma 1) of \cite{Zoghi+14RUCB}.
	Suppose $\cG_{ij}(t)$ denotes the event that at time $t \in [T]$ and  item-pair $i,j \in [K]$,  $p_{ij} \in [l_{ij}(t), u_{ij}(t)]$. We also define $\cG^c_{ij}(t)$ its complement. Let $i,j \in [K]$. 

	Note that for any such that pair $(i,i)$, $\cG_{ii}(t)$ always holds true for any $t \in [T]$ and $i \in [n]$, as $p_{ii} = u_{ii} = l_{ii} = \frac{1}{2}$. We can thus assume $i \neq j$. Moreover, for any $t$ and $i,j$, $\cG_{ij}(t)$ holds if and only if $\cG_{ij}(t)$ as $|\hp_{ji}(t) - p_{ji}| = |(1-\hp_{ij}(t)) - (1-p_{ij})| = |\hp_{ij}(t) - p_{ij}| $.
	Thus we will restrict our focus only to pairs $i < j$ for the rest of the proof. Hence, to prove the lemma it suffices to show
	\[
		 \P \Big( \exists t \in [T],i < j, \text{ such that } ~\cG^c_{ij}(t) \Big) \leq \delta \, \,,
	\]
	which we do now.
	Recall from the definition of $c_{ij}(t)$ that $\cG_{ij}(t)$ can be rewritten as:
	\[
		|\hp_{ij}(t) - p_{ij}| \le \sqrt{\frac{\alpha\ln (a_{ij}(t))}{n_{ij}(t)}} \,.
	\] 
	Let $\tau_{ij}(n)$ the time step $t \in [T]$ when the pair $(i,j)$ was updated (i.e. $i$ and $j$ was compared) for the $n^{th}$ time. 
	%
	%Clearly for any $n \in \N$, $\tau_{ij}(n+1) \ge \tau_{ij}(n)$ and if $\cG_{ij}$ holds at time $\tau_{ij}(n)$, it should hold at time $\tau_{ij}(n+1)$ as well, just by the definition of $c_{ij}(t)$. 
	%
	We now bound the probability of the confidence bound $(\cG_{ij}(t))$ getting violated at any round $t \in [T]$ for some duel $(i,j)$ as follows:
	\begin{align*}
		 \P \Big( & \exists t \in [T],i<j, \text{ such that } ~\cG^c_{ij}(t) \Big) \le \sum_{i < j}  \P \Bigg( \exists n \ge 0, |p_{ij} - \hp_{ij}(\tau_{ij}(n)) | > \sqrt{\frac{\alpha a_{ij}(\tau_{ij}(n))}{n_{ij}(\tau_{ij}(n))}} \Bigg) \\
		& = \sum_{i < j} \Bigg [ \P \Bigg( \exists n \le C(K,\delta), ~|p_{ij} - \hp_{ij}(n) | > \sqrt{\frac{\alpha  \ln (C(K,\delta)) }{n}} \Bigg)\\
		& \hspace{2cm} + \P \Bigg( \exists n > C(K,\delta), ~|p_{ij} - \hp_{ij}(\tau_{ij}(n)) | > \sqrt{\frac{\alpha \ln \big(n_{ij}\big(\tau_{ij}(n) \big)\big)}{n}} \Bigg) \Bigg ],
	\end{align*}
	where $\hp_{ij}(t) = \frac{w_{ij}(t)}{w_{ij}(t) + w_{ij}(t)}$ is the frequentist estimate of $p_{ij}$ at round $t$ (after $n = n_{ij}(t)$ comparisons between arm $i$ and $j$). 
	To ease the notation, denote $F = C(K,\delta)$. Noting $n_{ij}(\tau_{ij}(n))= n$, and using Hoeffding's inequality, we further get
	\begin{align*}
		\P \Big( \exists t \in [T], & i<j, \text{ such that } ~\cG^c_{ij}(t) \Big)
		\le \sum_{i < j} \Bigg [ \sum_{n=1}^{F}2e^{-2n\frac{\alpha\ln F}{n}} + \sum_{n=F+1}^{\infty}2e^{-2n\frac{\alpha\ln n}{n}} \Bigg] \\
		& = \frac{n(n-1)}{2}\Bigg [ 2\sum_{n = 1}^{F}\frac{1}{F^{2\alpha}} + \sum_{n = F+1}^{\infty}\frac{2}{n^{2\alpha}} \Bigg]\\
		& \le \frac{n^2}{F^{2\alpha-1}} + n^2\int_{F}^{\infty}\frac{dx}{x^{2\alpha}} 
		\le \frac{n^2}{F^{2\alpha-1}} - \frac{n^2}{(1-2\alpha)F^{2\alpha-1}}
		= \frac{(2\alpha)n^2}{(2\alpha-1)F^{2\alpha-1}} = \delta. 
	\end{align*}
where the last inequality is because $F = C(K,\delta) = \Big[ \frac{2\alpha n^2}{(2\alpha-1)\delta} \Big]^{\frac{1}{2\alpha-1}}$. This concludes the claim.

\end{proof}

\begin{customlemma}{\ref{lem:rucb_nijs}}
	Let $\alpha > 0.5$. Under the notations and the high-probability event of Lem.~\ref{lem:conf_cdels}, for all $i,j,k \in [K]$ such that $\{i,j\} \neq \{k,k\}$, and for any $\tau \geq 1$
	\begin{equation*}
		\sum_{t = 1}^{\tau} \1(i_t^* = k)\1\big(\{x_t,y_t\} = \{i,j\}\big) 
		\le \frac{4\alpha \log a_{i,j}(\tau)}{\min\big\{\Delta(k,i)_+ ,\Delta(k,j)_+ \big\}^2} \,,
	\end{equation*}
where $a_{ij}(\tau) = \max\big( C(K,\delta),n_{ij}(\tau)\big)$.
\end{customlemma}

\begin{proof}
We assume the confidence bound of Lem.~\ref{lem:conf_cdels} is holds good for all pair $(i,j) \in [K]^2$, at all round $t \in [T]$, which we know happens with probability at least $(1-\delta)$. Let us define $l_{ij}(t):= 1-u_{ji}(t)$. 
Let $t\geq 1$. Let $i,j,k \in [K]$ such that $i_t^* = k$, $x_t = i$, and $y_t = j$ and $\{i,j\} \neq \{k,k\}$. Since $i_t^* = k$, this implies both $i \geq k$ and $j \geq k$. Furthermore, we recall that $i_t^* = k$ is unique by definition, $i_t^* = \min\{S_t\}$. We consider the following cases.

\begin{itemize}[leftmargin=15pt]
	\item \textbf{Case 1 ($i = j > k$).} Then, $x_t = y_t = i = j$. 
	By the arm selection strategy (Step 14. of Algorithm~\ref{alg:rucbp}) 
	\[
		y_t \leftarrow \arg\max_{m \in \cC_t}u_{m x_t}(t) \,\,
	\] 
	which implies $1/2 = u_{jj}(t) > u_{kj}(t)$. 
	But, on the other hand, since $k < j$, by Lemma~\ref{lem:conf_cdels}, $u_{kj}(t) \ge p_{kj} > \frac{1}{2}$. This causes a contraction and this case is not possible.

	\item \textbf{Case 2 ($j > i = k$). } Then, $y_t =j$ and $x_t = i_t^* = k$.  We again proceed by contradiction. Assume that ${n_{kj}(t) > \frac{4\alpha \ln a_{kj}(t)}{\Delta(k,j)_+^2}}$.  Then, by definition of $c_{kj}(t)$, it implies 
	\[
		2c_{kj}(t) = 2\sqrt{\frac{\alpha \log a_{kj}(t)}{n_{ij}(t)}} < \Delta(k,j)_+ \,,
	\]
	which by Lem.~\ref{lem:conf_cdels} entails 
	\[
		u_{jk}(t) = \hp_{jk}(t) + c_{jk}(t) < p_{jk} + 2c_{jk}(t) < \frac{1}{2} - \Delta(k,j)_+ + \Delta(k,j)_+ < \frac{1}{2} \,.
	\] 
	Again since our arm selection strategy enforces $y_t \leftarrow \arg\max_{i \in \cC_t}u_{i x_t}(t)$, clearly $\frac{1}{2} = u_{kk}(t) > u_{jk}(t)$, so that $j$ can not be selected as $y_t$. Therefore, recalling that $k=i$,
	\begin{equation}
		\label{eq:case2}
		n_{ij}(t) \leq \frac{4\alpha \ln a_{kj}(t)}{\Delta(i,j)_+^2} \,.
	\end{equation}

	\item \textbf{Case 3 ($i > j = k$).} Then, $x _t =i$ and $y_t = i_t^* = k$.
	This can be proved similarly as the previous case. Assuming $n_{ik}(t) > {4\alpha \ln a_{ik}(t)} \Delta(i,k)_+^{-2}$ yields $u_{ik}(t) < \nicefrac{1}{2}$. Therefore, since $u_{ii}(t) = \nicefrac{1}{2}$, it entails
	\[
		|C_i(t)| = \big|\big\{m \in S_t| u_{im}(t) > \nicefrac{1}{2}\big\} \big| \leq \big|S_t \backslash \{i,k\}\big| \leq |S_t| - 2\,.
	\]
	But by Lemma~\ref{lem:conf_cdels}, for all $m > k$, $u_{km}(t)  \geq  p_{km} = \nicefrac{1}{2} + \Delta(k,m) > \nicefrac{1}{2}$. Thus, since $k = i_t^*$, we also have $|C_k(t)| = |S_t| - 1$ and thus
	\[
		|C_k(t)| > |C_i(t)|\,.
	\]
	By Step 12 of Algorithm~\ref{alg:rucbp}, this implies that $i \notin C_t$ and thus $x_t \neq i$ as $x_t$ is selected from $C_t$, which causes a contradiction. Therefore, recalling $j = k$,
	\begin{equation}
		\label{eq:case3}
		n_{ij}(t) \leq  \frac{4\alpha \ln a_{i,j}(t)}{\Delta(i,j)_+^2} \,.
	\end{equation}

	\item \textbf{Case 4. ($i \neq j > k$).} 
	Then, assuming $n_{ij}(t) >  4\alpha \log a_{i,j}(t) \min \{\Delta(k,i)_+,\Delta(k,j)_+\}^{-2}$, note that
	\[
		u_{ij}(t) - l_{ij}(t) = 2c_{ij}(t) =  2\sqrt{\frac{\alpha \log a_{i,j}(t)}{n_{ij}(t)}} <  \min (\Delta(k,i)_+,\Delta(k,j)_+) \,.
	\]

	But, on the other hand, $x_t = i$ implies $u_{ij}(t) > \nicefrac{1}{2}$, and $y_t = j$ implies $u_{ji}(t) > u_{jk}(t) > p_{jk}$, and then $l_{ij}(t) = 1-u_{ji}(t) < 1- p_{jk}$. So we have $u_{ij}(t) - l_{ij}(t) > \nicefrac{1}{2} - p_{ik} = \Delta(k,i)_+$ which gives a contradiction. Thus,
	\begin{equation}
		\label{eq:case4}
		n_{i,j}(t) \leq \frac{4\alpha \log a_{i,j}(t)}{\min\big\{\Delta(k,i)_+ ,\Delta(k,j)_+ \big\}^2} \,.
	\end{equation}
\end{itemize}
Note that the case $x_t=j$, $y_t = i$, and $i_t^* = k$ is symmetric with the above cases and can be considered similarly. Denote by $\tau'$ the last time before $\tau \geq 1$ such that a pair $\{i,j\} \neq \{k,k\}$ is pulled when $k = i_t^*$, that is
\[
	\tau' = \argmax_{1\leq t \leq \tau} \big\{ k = i_t^*,  \{x_t,y_t\} = \{i,j\}\big\} \,.
\]
Then, 
\begin{align*}
	\sum_{t = 1}^{\tau} \1(i_t^* = k)\1\big(\{x_t,y_t\} = \{i,j\}\big) \leq \sum_{t = 1}^{\tau'} \1(i_t^* = k)\1\big(\{x_t,y_t\} = \{i,j\}\big) \leq  n_{ij}(\tau') \,.
\end{align*}
But, at time $\tau'$, the suboptimal pair $\{i,j\}$ got pulled, thus one of the above four cases is true, which implies from Inequalities~\eqref{eq:case2},~\eqref{eq:case3}, and~\eqref{eq:case4} that
\[
	n_{ij}(\tau')  \leq \frac{4\alpha \log a_{i,j}(\tau)}{\min\big\{\Delta(k,i)_+ ,\Delta(k,j)_+ \big\}^2} \,.
\]
Substituting into the previous inequality concludes the proof. 
\end{proof} 

% % % % % % % % % % % % % % % % % % % % % % % % % % %
	
%!TEX root = nips20_sleepingDB.tex
\section{Appendix for Sec.~\ref{sec:algo_komiyama}}
\label{app:rmed}

% \subsection{Proof of Lem.~\ref{lem:goodevent}}
\subsection{Technical lemmas}
Before proving the regret guarantee of \algrmed ~(Alg.~\ref{alg:rmed}) in Thm.~\ref{thm:ub_rmed}, we would like to introduce three lemmas which are crucially used towards bounding Alg.~\ref{alg:rmed}'s regret. Lemma~\ref{lem:goodevent} below states that after some exploration, the algorithm estimates well all $p_{ij}$ with $\hp_{ij}(t)$. 

\begin{restatable}[]{lem}{lemgood}
\label{lem:goodevent}
Let $t_0 \geq 1$ and $\alpha \geq 4K$. Let $\epsilon_i > 0$ for all $i \in [K]$. 
For any $t \in [T]$, let us define $\cE_t:= \{ n_{ij}(t) > t_0, \forall i,j \in S_t, i \neq j\}$ to be the event when all distinct pairs $i,j \in [K]$ is played for at least for $t_0$ times. 
Let us also denote the event $\cG(t) := \{\forall j > i_t^*, \Delta(i_t^*,j) > \epsilon_i, \, \hp_{i_t^* j}(t) > \nicefrac{1}{2}\}$. $\cG^c(t)$ denotes the complement event of $\cG(t)$.
Then \algrmed\, satisfies:
\[
\E\bigg[\sum_{t = 1}^{T}\1(\cE_t)\1\big(\cG^c(t)\big)\bigg] = 2K + K \sum_{i = 1}^{K-1}\sum_{j = i+1 \mid \Delta(i,j)> \epsilon_i}^{K} \frac{e^{-(t_0-1)\Delta(i,j)^2} }{\Delta(i,j)^2}  \,. 
\]
\end{restatable}

\begin{proof}
First, we show that with high probability for all $t=1,\dots,T$, $i_t^* \in \cC_t$, $i_t^*$ belongs to the set of potential winners. Let $t\geq 1$. By definition of $\cC_t$, we have
\begin{align*}
	\P(i^*_t \notin \cC_t) 
		& \leq \P\Big(\sum_{j \in \hat \cB_{i^*_t}(t)} n_{i^*_tj}(t) \kl(\hat p_{i^*_tj}(t),0.5) \geq \alpha \log t \Big) \\
		& \leq \P\Big(\exists j>i^*_t \quad \text{s.t.}\quad  n_{i^*_tj}(t) \kl(\hat p_{i^*_tj}(t),0.5) \geq \frac{\alpha \log t}{K} \Big) \\
		& \leq \sum_{j = i^*_t+1}^K  \sum_{n = 1}^t \P\Big(   \, \kl(\tilde  p_{i^*_tj}(n), 0.5) \geq \frac{\alpha \log t }{n K} \Big) \,,
\end{align*}
where $\tilde p_{ij}(n)$ denotes the frequentist empirical estimate of $P(i,j)$ after $n$ pairwise comparisons between $i$ and $j$ (i.e., $\tilde p_{ij}(n) = \hat p_{ij}(t)$ with $n_{ij}(t) = n$). From Lemma II.1 of \cite{csiszar1998method}, this yields
\[
 \P(i^*_t \notin \cC_t) \leq \sum_{j = i^*_t+1}^K  \sum_{n = 1}^t n \exp\Big(- \frac{\alpha \log t}{K}\Big) \leq Kt^2 \exp\Big(- \frac{\alpha \log t }{K}\Big) \leq \frac{K}{t^2} \,,
\]
since $\alpha \geq 4K$. Therefore, 
\begin{equation}
	\label{eq:istarnotinCt}
	\sum_{t=1}^T \P(i^*_t \notin \cC_t) \leq K \sum_{t=1}^T \frac{1}{t^2} \leq 2K \,.
\end{equation}
Then, 
\begin{align*}
	\E\bigg[\sum_{t = 1}^{T}\1(\cE_t)\1\big(\cG^c(t)\big)\bigg] 
	& \le  \E\bigg[\sum_{t = 1}^{T}\1(\cE_t)\1\big(\cG^c(t)\big) \1(i_t^* \in \cC_t)\bigg] + \E\bigg[\sum_{t = 1}^{T}\1(i_t^* \notin \cC_t)\bigg] \\
	& \stackrel{\eqref{eq:istarnotinCt}}{\leq} \E\bigg[\sum_{t = 1}^{T}\1(\cE_t)\1\big(\cG^c(t)\big) \1(i_t^* \in \cC_t)\bigg] + 2K \\
	& = \E\bigg[\sum_{t = 1}^{T} \sum_{i=1}^K \1(\cE_t)\1\big(\cG^c(t)\big) \1(i_t^* = i) \1(i \in \cC_t)\bigg] + 2K \,.
\end{align*}
Now, since $x_t$ is uniformly sampled from $\cC_t$ from Line $11$ of Algorithm~\ref{alg:rmed}, given that $i \in \cC_t$, the probability that $x_t = i$ is at least $1/|\cC_t| \geq 1/K$. Thus,
\[
\E\big[\1(\cE_t)\1\big(\cG^c(t)\big) \1(i_t^* = i) \1(i \in \cC_t) \big] \leq K \E\big[\1(\cE_t)\1\big(\cG^c(t)\big) \1(i \in \cC_t) \1(i_t^* = x_t = i)\big] \,,
\]
which yields
\begin{align}
	\E&\bigg[\sum_{t = 1}^{T}\1(\cE_t)\1\big(\cG^c(t)\big)\bigg] 
		 \leq K  \sum_{i=1}^K \E\bigg[ \sum_{t=1}^T \1(\cE_t)\1\big(\cG^c(t)\big) \1(i \in \cC_t) \1(i_t^* =  x_t = i ) \bigg] + 2 K  \nonumber \\
		& \stackrel{(*)}{=} K  \sum_{i=1}^K \sum_{j:\Delta(i,j)>\epsilon_i}  \E\bigg[ \sum_{t=1}^T \1(\cE_t) \1(i \in \cC_t) \1(i_t^* =  i ) \1\big((x_t,y_t) = (i,j)\big) \1( \hat p_{ij}(t) < 1/2) \bigg] + 2 K  \nonumber \\
		& \leq   K  \sum_{i=1}^K \sum_{j:\Delta(i,j)>\epsilon_i}  \E\bigg[ \sum_{t=1}^T\1(\cE_t)  \1\big((x_t,y_t) = (i,j)\big) \1( \hat p_{ij}(t) < 1/2) \bigg] + 2 K  \label{eq:previous} 
\end{align}
where $(*)$ is because $\cG^c(t) := \big\{ \exists j > i, \Delta(i,j) > \epsilon_{i_t^*}, \hat p_{ij}(t) < 1/2\big\}$ when $i = i_t^*$, and since $y_t$ is chosen such that $\hat p_{i y_t}(t) < 1/2$ (see Line 12. of Alg.~\ref{alg:rmed}). Recall that $\cE_t$ ensures that $(i,j)$ was pulled at least $t_0$ times during the exploration phase. Recalling that $\tilde p_{ij}(n)$ equals $\hat p_{ij}(t)$ where $t$ is such that $n = n_{ij}(t)$, we have
\begin{align*}
 \sum_{t=1}^T\1(\cE_t)  \1\big((x_t,y_t) = (i,j)\big) \1( \hat p_{ij}(t) < 1/2)\leq \sum_{n=t_0}^\infty   \1( \tilde p_{ij}(n) < 1/2) \,.
\end{align*}
Therefore, plugging the latter inequality into the previous upper-bound~\eqref{eq:previous}, it yields
\begin{align*}
	\E\bigg[\sum_{t = 1}^{T}\1(\cE_t)\1\big(\cG^c(t)\big)\bigg] 
		& \leq K  \sum_{i=1}^K \sum_{j:\Delta(i,j)>\epsilon_i}  \sum_{n=t_0}^\infty \P\big( \tilde p_{ij}(n) < 1/2\big) + 2K \\
		& \stackrel{(a)}{=}  K  \sum_{i=1}^K \sum_{j:\Delta(i,j)>\epsilon_i}  \sum_{n=t_0}^\infty \P\big( \tilde p_{ij}(n) < P(i,j) - \Delta(i,j) \big) + 2K \\ 
		& \stackrel{(b)}{\leq}  K  \sum_{i=1}^K \sum_{j:\Delta(i,j)>\epsilon_i}  \sum_{n=t_0}^\infty \exp\big(-n \Delta(i,j)^2\big) + 2K \\
		& \leq  K  \sum_{i=1}^K \sum_{j:\Delta(i,j)>\epsilon_i} \frac{e^{-(t_0 - 1) \Delta(i,j)^2}}{e^{\Delta(i,j)^2}-1} + 2K  \\
		& \leq K  \sum_{i=1}^K \sum_{j:\Delta(i,j)>\epsilon_i} \frac{e^{-(t_0 - 1) \Delta(i,j)^2}}{\Delta(i,j)^2} + 2K 
\end{align*}
where  $(a)$ follows by definition $\Delta(i,j) := \P(i,j) - 1/2$ and $(b)$ is by Hoeffding's inequality. 
\end{proof}

The high-level idea of Lemma~\ref{lem:ni} below is that for any pair $1\leq i<j \leq K$, $j$ will not be played too much more than $M_{ij}(\delta)$ times together with items $k\leq i$. In other words, after sufficiently enough rounds $j$ is detected as worse than all items $k<i$. 

\begin{restatable}[]{lem}{lemni}
\label{lem:ni}
Let $1\leq i < j \leq K$. Then, \algrmed ~(Alg.~\ref{alg:rmed}) satisfies:
\[
 \E\bigg[ \sum_{t = 1}^T  \1(\cG(t)) \sum_{k=1}^i \1(x_t = j, y_t = k) \1\Big(N_{ij}(t) > M_{ij}(\delta)\Big) \bigg]  \leq \frac{32}{\delta^2 \Delta(i,j)^2},
\]
where $\cG(t)$ is as defined in Lem.~\ref{lem:goodevent} and $N_{ij}(t) := \sum_{k=1}^i n_{kj}(t)$ is the number of times $j$ was compared with some arm in $1,\dots,i$ and $M_{ij}(\delta):= {(\alpha+\delta)(\log T)}/{\kl(\p_{ji},0.5)}$.
\end{restatable}

\begin{proof}
Let $1\leq i \leq K-1$. We start by recalling some useful notations: 
\[
\hcB_i(t):= \Big\{j \mid j \in [K], \hat p_{i,j}(t) \leq \nicefrac{1}{2}\Big\} \,, \qquad \cI_i(t):= \sum_{j \in \hcB_i(t)} n_{ij}(t)\kl(\hat p_{ij}(t),0.5) \,,
\]
where $\hi^*(t):= \arg\min_{i \in [K]}\cI_i(t)$, and for simplicity we here denote $\cI_{\hi^*_t}(t) = \cI^*(t):= \min_{i \in [K]}\cI_i(t)$. We also denote that event $\cJ_i(t):= \{\cI_i(t) - \cI^*(t) \le \alpha \log t\}$.  
Then for any fixed $j>i$, we have
\begin{align*}
S_T(i,j) & :=  \E\bigg[\sum_{t =1}^T \1(\cG(t)) \sum_{k =1}^i \1(x_t = j, y_t = k)  \1\Big( N_{ij}(t) > M_{ij}(\delta)\Big)\bigg] \\
& = \E\bigg[\sum_{t = 1}^{T} \1(\cG(t)) \sum_{k =1}^i  \1(x_t = j, y_t = k) \1\Big( N_{ij}(t) > M_{ij}(\delta) \Big)\1(\cJ_j(t) \Big)\bigg]  
\hspace*{.5cm} \leftarrow \ \text{as } x_t = j \text{ implies } \1(\cJ_j(t))=1\\
& =\E\bigg[ \sum_{t = 1}^{T} \1(\cG(t)) \sum_{k =1}^i  \1(x_t = j, y_t = k) \1\Big( N_{ij}(t) > M_{ij}(\delta), \cJ_j(t) \Big)\bigg]
\end{align*}
Substituting $\cJ_i(t):= \{\cI_i(t) - \cI^*(t) \le \alpha \log t\}$, and using that $\cG(t)$ implies $\cI^*(t) = 0$, we get
\begin{align*}
	S_T(i,j) 
		& \leq \E\bigg[  \sum_{t=1}^T \sum_{k=1}^i \1(x_t = j, y_t = k) \1\Big( N_{ij}(t) > M_{ij}(\delta), \cI_j(t) \le \alpha \log t \Big)\bigg] \\
		& \leq  \E\bigg[  \sum_{t=1}^T \sum_{k=1}^i \1(x_t = j, y_t = k) \1\Big( N_{ij}(t) > M_{ij}(\delta), \sum_{k \in \hcB_j(t)}n_{jk}(t)\kl(\hp_{jk}(t) , 0.5 ) \le \alpha \log  T \Big)\bigg] \\
		& \leq  \E\bigg[  \sum_{t=1}^T \sum_{k=1}^i \1(x_t = j, y_t = k) \1\Big( N_{ij}(t) > M_{ij}(\delta), \sum_{k =1}^i n_{kj}(t) \kl^+(\hp_{jk}(t) , 0.5 ) \le \alpha \log  T \Big)\bigg] 
\end{align*}
where $\kl^+(p ,q):= \kl(p,q)\1(p<q)$. But, from convexity of $\kl^+(\cdot, 0.5)$ together with Jensen's inequality
\[
	\sum_{k =1}^i n_{kj}(t) \kl^+(\hp_{jk}(t) , 0.5 ) \geq  N_{ij}(t) \kl^+\bigg(\frac{1}{N_{ij}(t)} \sum_{k =1}^i n_{kj}(t) \hp_{jk}(t) , 0.5 \bigg)\,.
\]
Therefore, denoting 
\[
\tilde p_{1:ij}(N_{ij}(t)) := \frac{1}{N_{ij}(t)} \sum_{k =1}^i n_{kj}(t) \hp_{jk}(t) = \frac{1}{N_{ij}(t)}  \sum_{k=1}^i w_{ki}(t)
\] 
the frequentist empirical estimate obtained after $N_{ij}(t)$ comparisons of $j$ with any item better than $i$, we have
\[
	S_T(i,j) \leq \E\bigg[  \sum_{t=1}^T \sum_{k=1}^i \1(x_t = j, y_t = k) \1\Big( N_{ij}(t) > M_{ij}(\delta), \  N_{ij}(t)  \kl^+\big(\tilde p_{1:ij}(N_{ij}(t)) , 0.5 \big) \le \alpha \log  T \Big)\bigg] 
\]
But, for each $n>M_{ij}(\delta)$, $N_{ij}(t)=n$ is only possible for one of the above rounds since $(x_t,y_t) = (i,k)$ with $k\leq i$, which increases $N_{ij}(t)$ by one. Thus, 
\begin{align*}
	S_T(i,j) & \leq \E\bigg[  \sum_{n=M_{ij}(\delta)}^T \1\Big(n \kl^+\big(\tilde p_{1:ij}(n) , 0.5 \big) \le \alpha \log  T \Big)\bigg] \\
		& \leq   \E\Bigg[\sum_{n = \lceil M_{ij}(\delta)\rceil }^{T}\1\Bigg( M_{ij}(\delta)  \kl^+(\tilde p_{1:ij}(n), 0.5)  \le \alpha \log T \Bigg)\Bigg] \\
	&\leq  \E\Bigg[\sum_{n = \lceil M_{ij}(\delta)\rceil }^{T}\1\Bigg(\kl^+(\tilde p_{1:ij}(n), 0.5) \le \frac{\kl(p_{ji} , 0.5)}{1+\delta} \Bigg)\Bigg]
\end{align*}
Now, let $\mu_i \in (p_{ji},0.5)$ such that $\kl(\mu_i, 0.5) = \kl(p_{ji},0.5)/(1+\delta)$. 
By monotonicity of $\kl^+(\cdot,0.5)$,
\begin{align*}
S_T(i,j) & = \sum_{n = \lceil M_{ij}(\delta)\rceil }^{T}\P\Big( \kl^+( \tilde p_{1:ij}(n), 0.5) \le \kl^+(\mu_i , 0.5) \Big) \\
& \le \sum_{n = \lceil M_{ij}(\delta)\rceil }^{T}\P\Big(\tilde p_{1:ij}(n) \le \mu_i \Big) \\
& \le \sum_{n = \lceil M_{ij}(\delta)\rceil }^{T}\P\Big(\tilde p_{ij}(n) \le \mu_i \Big) \\
& \le \sum_{n = \lceil M_{ij}(\delta)\rceil }^{T} e^{-\kl(\mu_i,p_{ij})n}  \,,
\end{align*}
where the last inequality is by Chernoff's inequality (e.g. see Fact 8 of \cite{Komiyama+15}). 
Then,
\[
S_T(i,j) \leq \sum_{n=1}^\infty e^{-\kl(\mu_i,p_{ij})n}  \le \frac{1}{\kl(\mu_i,p_{ij})} \,.
\]
The proof is concluded using Pinksker's inequality followed by $4 \Delta(i,j)$-Lipschitzness of $\kl(\cdot,0.5)$ over $(0.5 - \Delta(i,j), 0.5)$:
\begin{align*}
  \kl(\mu_i,p_{ij}) 
  	& \geq 2 (\mu_i - p_{ij})^2  \hspace*{4.5cm} \leftarrow \text{Pinsker's inequality} \\
  	& \geq \frac{2}{16 \Delta(i,j)^2} \big(\kl(\mu_i,0.5) - \kl(p_{ij},0.5)\big)^2  \hspace*{1cm} \leftarrow \text{Lipschitzness} \\
  	& = \frac{2 \kl(p_{ij},0.5)^2 \delta^2}{ 16 \Delta(i,j)^2 (1+\delta)^2}  \hspace*{4cm} \leftarrow \text{def of $\mu_i$} \\
  	& \geq \frac{\kl(p_{ij},0.5)^2 \delta^2}{32 \Delta(i,j)^2} \hspace*{4cm} \leftarrow \delta \in (0,1) \\
  	& \geq \frac{\Delta(i,j)^2 \delta^2}{32} \,. \hspace*{4cm} \leftarrow \text{Pinsker's inequality} 
\end{align*}
Therefore,
\[
	S_T(i,j) \leq \frac{32}{\Delta(i,j)^2 \delta^2} \,.
\]
\end{proof}

\begin{restatable}[]{lem}{lemgap}
\label{lem:gap}
For any $\epsilon_2,\dots,\epsilon_K \ge 0$,
\begin{align*}
& \sum_{1 \le i<j \le K \mid \Delta(i,j) > \epsilon_j}\frac{\Delta(i,j) - \Delta(i+1,j)}{\Delta(i,j)^2} \le  \sum_{j = 2}^{n}\frac{2}{\max\big\{ \epsilon_j, \Delta(j-1,j) \big\}}
%& \sum_{1 \le i<j \le n \mid \Delta(i,j) > \epsilon}\frac{\Delta(i,j-1) - \Delta(i,j)}{\Delta(i,j)^2} \le 2 \sum_{i = 1}^{j_0(n)-1}\frac{1}{\max\Big( \epsilon, \Delta(j_0(i),j_0(i)+1) \Big)}
\end{align*}
%\red{the expressions might be a bit different than Klienberg's Lem. 5, but I believe it should hold good even outside plackett-luce assumptions, will check.}
\end{restatable}

\begin{proof}
The proof is adapted from similar techniques used for proving Lem. $5$ of \cite{kleinberg+10}.
First note that 
\begin{align*}
\sum_{1 \le i<j \le n \mid \Delta(i,j) > \epsilon_j}\frac{\Delta(i,j) - \Delta(i+1,j)}{\Delta(i,j)^2} = \sum_{i = 1}^{K-1}\sum_{j \in [K]\sm [i] \mid \Delta(i,j) > \epsilon_j} \frac{\Delta(i,j) - \Delta(i+1,j)}{\Delta(i,j)^2}
\end{align*}

Let us fix any arm $i \in [K-1]$, and denote by $\nabla_{i,j} := \Delta(i,j) - \Delta(i+1,j)$. Then we note
\begin{align*}
\sum_{j \in [K]\sm [i] \mid \Delta(i,j) > \epsilon_j}& \frac{\Delta(i,j) - \Delta(i+1,j)}{\Delta(i,j)^2} = \sum_{j \in [K]\sm [i] \mid \Delta(i,j) > \epsilon_j} \nabla_{i,j}\int_{0}^\infty \1(\Delta(i,j)^{-2} \ge x )dx\\
& = \sum_{j = i+1}^K\1(\Delta(i,j) > \epsilon_j) \nabla_{i,j}\int_{0}^\infty \1(\Delta(i,j)^{-2} \ge x )dx\\
& = \sum_{j = i+1}^K \nabla_{i,j}\int_{0}^\infty  \1(\Delta(i,j) > \epsilon_j, \Delta(i,j)^{-2} \ge x )dx\\
& = 2\sum_{j = i+1}^K \nabla_{i,j}\int_{0}^\infty  y^{-3}\1(\epsilon_j < \Delta(i,j) < y )dy \hspace*{.2cm} \leftarrow \text{ change of variable } x = y^{-2}, dx = -2 y^{-3} dy\\
& = 2\sum_{j = i+1}^K \nabla_{i,j}\int_{\epsilon_j}^\infty  y^{-3}\1(\epsilon_j < \Delta(i,j) < y )dy
\end{align*}

Further summing over all $i \in [K-1]$, we get
\begin{align*}
A_T & := \sum_{1 \le i<j \le n \mid \Delta(i,j) > \epsilon_j}\frac{\Delta(i,j) - \Delta(i+1,j)}{\Delta(i,j)^2} \\
& = \sum_{i = 1}^{K-1}\bigg( 2\sum_{j = i+1}^K \nabla_{i,j}\int_{\epsilon_j}^\infty  y^{-3}\1(\epsilon_j < \Delta(i,j) < y )dy \bigg)\\
& = 2 \sum_{i = 1}^{K-1}\sum_{j = i+1}^K  \int_{\epsilon_j}^\infty \bigg(\nabla_{i,j} y^{-3}\1(\epsilon_j < \Delta(i,j) < y )dy \bigg)\\
& = 2\sum_{j = 2}^K \sum_{i = 1}^{j-1} \int_{\epsilon_j}^\infty y^{-3}  \bigg(\big( \Delta(i,j) - \Delta(i+1,j) \big) \1(\epsilon_j < \Delta(i,j) \leq y) \bigg)dy\\
& = 2 \sum_{j = 2}^K  \int_{\epsilon_j}^\infty y^{-3} \sum_{i = i_y(j)}^{i_{\epsilon_j}(j)-1} \Big( \Delta(i,j) - \Delta(i+1,j) \Big) dy \,,
\end{align*}
where $i_\epsilon(j) := \arg\min\{i | i\leq j, \Delta(i,j)\leq \epsilon\}$ (with the convention that the sum is empty if the $\arg\min$ is empty) and because $\epsilon_j < \Delta(i,j) \leq  y$ is equivalent to $i_y(\epsilon) \leq i \leq i_{\epsilon_j} - 1$.  Using telescoping summation over $i$, we further get:
\begin{align*}
 A_T  & \leq 2 \sum_{j = 2}^K \int_{\epsilon_j}^\infty y^{-3} \big( \Delta(i_y(j),j) - \Delta(i_{\epsilon_j}(j),j) \big) dy \\
 	& \leq 2 \sum_{j = 2}^K  \int_{\epsilon_j}^\infty y^{-3} \Delta(i_y(j),j)  dy \hspace*{.4cm} \leftarrow \text{ since }  \Delta(i_{\epsilon_j}(j),j) > 0  \\
\end{align*}
Then, since $\Delta(i_{y}(j),j) = 0$ if $y < \Delta(j-1,j)$, we have
\begin{align*}
A_T & \leq 2 \sum_{j = 2}^K  \int_{ \max\{\epsilon_j, \Delta(j-1,j)\}}^\infty y^{-3} \Delta(i_y(j),j)  dy   \\
	 & \leq  2 \sum_{j = 2}^K  \int_{\max\{\epsilon_j, \Delta(j-1,j)\}}^\infty y^{-2}  dy \hspace*{2cm} \leftarrow \text{ since }   \Delta(i_y(j),j) \leq  y \nonumber \\
	 & \leq  2 \sum_{j = 2}^K \frac{1}{\max\{\epsilon_j,\Delta(j-1,j)\}} \,,
\end{align*}
which concludes the proof. 
\end{proof}

\subsection{Proof of Theorem~\ref{thm:ub_rmed}}

\ubrmed*

\begin{proof}
We analyse the expected regret \algrmed \,(Alg~\ref{alg:rmed}) for some fixed sequence $\cS_T$. Recall that $t_0$ is the budget spent on exploration of each pair $(i,j)$ and the notation
\[
	\cE(t) :=  \big\{ n_{ij}(t) > t_0, \forall i,j \in S_t, i \neq j\big\}\,,
\]
to be the event when all distinct pairs in $S_t$ have been explored $t_0$ times and 
\[
	\cG(t) := \big\{\forall j > i_t^*, \Delta(i_t^*,j) > \epsilon_i, \, \hp_{i_t^* j}(t) > \nicefrac{1}{2}\big\}
\]
the event when the probabilities $p_{ij}$ have been well estimated by the algorithm.  Then, from Lemma~\ref{lem:goodevent}, we have
\begin{align}
\E&\big[R_T\big]  = \E\bigg[\sum_{t =1}^{T}r_t \bigg]  
	 = \E\bigg[\sum_{t=1}^T \1(\cE^c(t)) r_t + \sum_{t = 1}^{T} \1(\cE(t)) \1(\cG^c(t))r_t + \sum_{t = T_0+1}^{T}  \1(\cE(t)) \1(\cG(t))r_t \bigg] \nonumber \\
	& \leq K^2 t_0 + 2K + K \sum_{i = 1}^{K-1}\sum_{j = i+1 \mid \Delta(i,j)> \epsilon_j}^{K} \frac{e^{-(t_0-1)\Delta(i,j)^2} }{\Delta(i,j)^2}  + \E\bigg[\underbrace{\sum_{t = T_0+1}^{T}  \1(\cE(t)) \1(\cG(t))r_t}_{E_T} \bigg] \,. \label{eq:regretdecomposition}
\end{align} 
We now upper-bound the third term of~\eqref{eq:regretdecomposition}. Remark that under $\cG(t)$ the algorithm chooses $y_t = i_t^*$. Therefore,
\begin{align}
E_T & := \sum_{t = 1}^{T}  \1(\cE(t)) \1(\cG(t)) r_t  \nonumber \\
	& \leq \sum_{t = 1}^{T}  \1(\cG(t)) r_t \nonumber\\
	& = \sum_{t=1}^T \1(\cG(t)) \sum_{1\leq i<j\leq K}  \1(x_t = j, y_t = i) r_t  \nonumber \\
	& = \sum_{t=1}^T \1(\cG(t)) \sum_{1\leq i<j\leq K} \1(x_t = j, y_t = i) \frac{\Delta(i,j)}{2}  \quad \leftarrow \ \text{because $\cG(t)$ implies $y_t = i_t^*$} \nonumber \\
	& \leq \sum_{1\leq i<j\leq K:\Delta(i,j)<\epsilon_j} n_{ij}(T) \frac{\Delta(i,j)}{2} +  \underbrace{\sum_{t=1}^T \1(\cG(t))  \sum_{1\leq i<j\leq K:\Delta(i,j)>\epsilon_j} \1(x_t = j, y_t = i) \frac{\Delta(i,j)}{2}}_{=: D_T} \label{eq:CT} 
 \end{align}
Moreover, recalling the notations $n_{ij}(T) := \sum_{t=1}^T \1\big(\{x_t,y_t\} = \{i,j\}\big)$  and defining 
\[
	\tilde N_{ij}(T) :=  \sum_{k=1}^i \sum_{s=1}^t  \1(\cG(s)) \1(x_t = k, y_t = j) \leq N_{ij}(T) := \sum_{k=1}^i n_{kj}(T) \,,
\] 
we have
\begin{align}
 D_T & :=  \sum_{1\leq i<j\leq K:\Delta(i,j)>\epsilon_j} \sum_{t=1}^T  \1(\cG(t)) \1(x_t = j, y_t = i) \frac{\Delta(i,j)}{2} \nonumber \\
	& = \sum_{1\leq i<j\leq K:\Delta(i,j)>\epsilon_j} \big(\tilde N_{ij}(T) - \tilde N_{(i-1)j}(T)\big) \frac{\Delta(i,j)}{2}  \nonumber \\
	& = \sum_{j=2}^K \sum_{i=1}^{i_{\epsilon_j}(j)}  \big(\tilde N_{ij}(T) - \tilde N_{(i-1)j}(T)\big) \frac{\Delta(i,j)}{2} \nonumber \\
	& = \sum_{j=2}^K \tilde N_{i_{\epsilon_j}j}(T) \frac{\epsilon_j}{2} + \sum_{j=2}^K \sum_{i=1}^{i_{\epsilon_j}(j)-1}  \tilde N_{ij}(T)  \frac{\Delta(i,j) - \Delta(i+1,j)}{2}  \nonumber \\
	& \leq \sum_{1\leq i <j \leq K: \Delta(i,j) \geq \epsilon_j} n_{ij}(T) \frac{\epsilon_j}{2} +  \sum_{1\leq i <j \leq K: \Delta(i,j) > \epsilon_j}  \tilde N_{ij}(T)  \frac{\Delta(i,j) - \Delta(i+1,j)}{2} \label{eq:DT} \,.
\end{align}
Now, we need to upper-bound $\tilde N_{ij}(T)$. We have,
\begin{align*}
	\tilde N_{i,j}(T) 
		& :=  \sum_{t=1}^T \1(\cG(t)) \sum_{k=1}^i \1\big(x_t = j,y_t = k\big)\\
		& \leq   \sum_{t=1}^T \1(\cG(t)) \sum_{k=1}^i \1\big(x_t = j,y_t = k\big) \Big[ \1\big(N_{ij}(t) \leq M_{ij}(\delta)\big) + \1\big(N_{ij}(t) > M_{ij}(\delta)\big)\Big] \\
		& \leq M_{ij}(\delta) + \sum_{t=1}^T \1(\cG(t)) \sum_{k=1}^i \1\big(x_t = j,y_t = k\big)\1\big(N_{ij}(t) > M_{ij}(\delta)\big) \\
		& \leq M_{ij}(\delta) + \frac{32}{\delta^2 \Delta(i,j)^2}  \hspace*{1cm} \leftarrow \text{Lemma~\ref{lem:ni}} \\
		& = \frac{(\alpha + \delta) \log T}{\kl(p_{ji},0.5)} +  \frac{32}{\delta^2 \Delta(i,j)^2} \\
		& \leq \Big(2(\alpha + \delta) \log T + \frac{32}{\delta^2}\Big) \frac{1}{\Delta(i,j)^2} \,,
\end{align*} 
where the last inequality comes from Pinksker's inequality. 
This entails
\begin{align*}
	 \sum_{1\leq i <j \leq K: \Delta(i,j) > \epsilon_j}  & \tilde N_{ij}(T)  \frac{\Delta(i,j) - \Delta(i+1,j)}{2}  \\
	 & \leq \Big((\alpha + \delta) \log T + \frac{16}{\delta^2}\Big)  \sum_{1\leq i <j \leq K: \Delta(i,j) > \epsilon_j} \frac{\Delta(i,j) - \Delta(i+1,j)}{\Delta(i,j)^2}  \\
	 & \leq 2 \sum_{j = 2}^{K}\frac{(\alpha + \delta) \log T + 16 \delta^{-2}}{\max\big\{ \epsilon_j, \Delta(j-1,j) \big\}} \,.
\end{align*}
Combining this inequality with~\eqref{eq:regretdecomposition},~\eqref{eq:CT}, and~\eqref{eq:DT} and choosing $t_0 = 1$ and $\alpha = 4K$ concludes
\begin{align*}
	\E\big[R_T] 
		& \leq K^2t_0 + 2K + K \sum_{i = 1}^{K-1}\sum_{j = i+1 \mid \Delta(i,j)> \epsilon_j}^{K} \frac{e^{-(t_0-1)\Delta(i,j)^2} }{\Delta(i,j)^2}   \\
		& \hspace*{2cm} +  \sum_{1\leq i <j \leq K} n_{ij}(T) \frac{\min\{\epsilon_j, \Delta(i,j)\}}{2} + 2 \sum_{j = 2}^{n}\frac{ (\alpha + \delta) \log T + 16 \delta^{-2}}{\max\big\{ \epsilon_j, \Delta(j-1,j) \big\}} \\
		& \leq K(K+2) 
		+ \sum_{1\leq i<j\leq K\mid \Delta(i,j)> \epsilon_j} \frac{K}{\Delta(i,j)^2 } \\
		& \hspace*{2cm}  +  \sum_{1\leq i <j \leq K} n_{ij}(T) \frac{\min\{\epsilon_j, \Delta(i,j)\}}{2} 
		+ 2 \sum_{j = 2}^{K}\frac{ (4K  + \delta) \log T + 16 \delta^{-2}}{\max\big\{ \epsilon_j, \Delta(j-1,j) \big\}}  \,.
\end{align*}
\end{proof}

\begin{proof}
Recall that the proof was done for any $\epsilon_2,\dots,\epsilon_K \geq 0$ that are independent of the algorithm. In particular, choosing $\epsilon_2,\dots,\epsilon_K = \epsilon$ entails that for any $\epsilon>0$
\begin{align*}
	\E\big[R_T\big] & \lesssim \  K^2 + 
		\epsilon T + \sum_{1\leq i<j\leq K\mid \Delta(i,j)> \epsilon} \frac{K}{\Delta(i,j)^2 }
		+ \sum_{j = 2}^{K}\frac{ K  \log T}{\max\big\{ \epsilon, \Delta(j-1,j) \big\}} 
\end{align*}
which yields making $\epsilon\to 0$ the distribution-dependent asymptotic upper-bound
\[
	\E\big[R_T] \leq O\bigg(K  \log (T) \sum_{j = 2}^{K} \frac{ \1\big\{\Delta(j-1,j)>0\big\} }{\Delta(j-1,j)}  \bigg)
\]
as $T \to \infty$ and for any fix $\epsilon \geq 0$ and choosing $\delta = 1$. Furthermore, optimizing $\epsilon_1 = \epsilon_2 = \dots = \epsilon_K = \epsilon = 2^{1/3}KT^{-1/3}$ yields the distribution-free upper-bound
\[
	\E\big[ R_T\big] \leq  K(K+2) + \frac{K^3}{\epsilon^2} + \frac{T \epsilon}{2} +  \frac{(8K  + 1) K \log T + 16K }{\epsilon} \leq 2 K T^{2/3} + O(K^2 + KT^{1/3}\log T) \,.
\]
\end{proof}

%%%%%%%%%%%%%%%%%%%%%%%%%

%\section{Conclusion and Future Perspective (detailed)}
%\label{app:concl}
%We introduce the problem of sleeping dueling bandits with stochastic preferences and adversarial availabilities, which, despite of great practical relevance, was left unaddressed till date. 
%
%Towards this we adapt two dueling bandit algorithms for the problem and give regret analysis for both. We also derive an instance dependent regret lower bound for our problem setup which shows that our second algorithm is asymptotically near-optimal (up to the problem dependent constants). Finally, we compare both our algorithms empirically where usually the first algorithm is shown to outperform the second, although having a relatively weaker regret.

%\textbf{Future Works.} Moving forward, one can address many open questions along this direction, including relaxing the \emph{total-ordering} assumption on the stochastic preferences assuming more general ranking objective based on {borda} \cite{SAVAGE} or {copeland} scores \cite{Zoghi+15}, or extending the framework to a general contextual scenario with subsetwise feedback. Another direction worth understanding is to analyze the connection of this problem with other bandit setups, e.g., learning with feedback graphs \cite{Alon+15,Alon+17} or other side information \cite{SideInfo11,SideInfo14}. It would also be interesting to consider the dueling bandit problem for adversarial preference and stochastic availabilities \cite{neu14,kanade09}, and also analyzing these class of problems for general subsetwise preferences \cite{Ren+18,sui2018advancements,Brost+16}.
}

\end{document}